\pdfoutput=1

\documentclass[sigconf]{acmart}

\AtBeginDocument{%
 \providecommand\BibTeX{{%
 \normalfont B\kern-0.5em{\scshape i\kern-0.25em b}\kern-0.8em\TeX}}}

\acmYear{2021}\copyrightyear{2021}
\setcopyright{acmcopyright}
\acmConference[MobiHoc '21]{The Twenty-second International Symposium on Theory, Algorithmic Foundations, and Protocol Design for Mobile Networks and Mobile Computing}{July 26--29, 2021}{Shanghai, China}
\acmBooktitle{The Twenty-second International Symposium on Theory, Algorithmic Foundations, and Protocol Design for Mobile Networks and Mobile Computing (MobiHoc '21), July 26--29, 2021, Shanghai, China}
\acmPrice{15.00}
\acmDOI{10.1145/3466772.3467038}
\acmISBN{978-1-4503-8558-9/21/07}

\usepackage{bm}
\usepackage[ruled,linesnumbered]{algorithm2e}
\usepackage{thmtools, thm-restate}
\usepackage{enumerate}
\usepackage{multirow}
\usepackage{subfigure}
\usepackage[switch]{lineno}
\usepackage{amsmath}

\DeclareMathOperator*{\argmin}{arg\,min}
\usepackage{url}

\newtheorem{assumption}{Assumption}

\newtheorem{definition}{Definition}
\newtheorem{lemma}{Lemma}
\newtheorem{theorem}{Theorem}
\newtheorem{corollary}{Corollary}

\newcommand{\etal}{\textit{et al}.}

\begin{document}
\frenchspacing
\allowdisplaybreaks[4]
\title{Inexact-ADMM Based Federated Meta-Learning for Fast and Continual Edge Learning}


\author{Sheng Yue}
\affiliation{%
	\institution{Central South University}
	\country{}
}
\affiliation{
    \institution{Arizona State University}
    \country{}
}
\email{sheng.yue@csu.edu.cn}

\author{Ju Ren}
\authornote{Corresponding author.}
\affiliation{%
	\institution{Central South University}
	\country{}
}
\email{renju@csu.edu.cn}

\author{Jiang Xin}
\affiliation{%
	\institution{Central South University}
	\country{}
}
\email{xinjiang@csu.edu.cn}

\author{Sen Lin}
\affiliation{%
	\institution{Arizona State University}
	\country{}
}
\email{sen.lin.1@asu.edu}

\author{Junshan Zhang}
\affiliation{%
	\institution{Arizona State University}
	\country{}
}
\email{junshan.zhang@asu.edu}



\begin{abstract}
 In order to meet the requirements for performance, safety, and latency in many IoT applications, intelligent decisions must be made right here right now at the network edge. However, the constrained resources and limited local data amount pose significant challenges to the development of edge AI. To overcome these challenges, we explore continual edge learning capable of leveraging the knowledge transfer from previous tasks. Aiming to achieve fast and continual edge learning, we propose a platform-aided federated meta-learning architecture where edge nodes collaboratively learn a meta-model, aided by the knowledge transfer from prior tasks. The edge learning problem is cast as a regularized optimization problem, where the valuable knowledge learned from previous tasks is extracted as regularization. Then, we devise an  ADMM based federated meta-learning algorithm, namely ADMM-FedMeta, where ADMM offers a natural mechanism to decompose the original problem into many subproblems which can be solved in parallel across edge nodes and the platform. Further,  a variant of inexact-ADMM method is employed where the subproblems are `solved' via linear approximation as well as Hessian estimation to reduce the computational cost per round to $\mathcal{O}(n)$. We provide a comprehensive analysis of ADMM-FedMeta,  in terms of the convergence properties, the rapid adaptation performance, and the forgetting effect of prior knowledge transfer, for the general non-convex case. Extensive experimental studies demonstrate the effectiveness and efficiency of ADMM-FedMeta and showcase that it substantially outperforms the existing baselines.

\end{abstract}





\maketitle

\section{Introduction}
The past few years have witnessed an explosive growth of Internet of Things (IoT) devices. In many of these IoT applications, decisions must be made in  real time to meet the requirements for safety, accuracy, and performance \cite{zhang2020data}. A general consensus is that the conventional cloud-based approach would not work well in these applications, calling for edge intelligence or edge AI \cite{zhou2019edge,park2019wireless}. Built on a synergy of edge computing and AI, edge intelligence is expected to push the frontier of model training and inference processes to the network edge in the physical proximity of IoT devices and data sources. Nevertheless, it is highly nontrivial for a single edge node to achieve real-time edge intelligence since AI model training usually requires extensive computing resources and a large number of data samples. To tackle these challenges, we resort to continual learning  capable of leveraging the knowledge transfer from previous tasks in the cloud  or by other edge nodes. Simply put, continual learning (CL)  is a machine learning paradigm  that is designed to sequentially learn
from data samples corresponding to different tasks \cite{parisi2019continual}. Rather than learning the new model from scratch, CL aims to design algorithms leveraging knowledge transfer from pre-trained models to the new learning task, assuming that the
training data of previous tasks are  unavailable for the newly coming task (this is the case for edge learning).

To facilitate edge learning, collaborative learning has recently been proposed to leverage the model knowledge distillation, including cloud-edge collaboration and edge-edge collaboration. More specifically, a distributionally robust optimization based edge learning framework has been introduced  to build a cloud-edge synergy between the pre-trained model in the cloud and the local data samples at the edge  \cite{zhang2020data}.  Along a different avenue, 
building on the recent advances in meta-learning \cite{finn2017model,nichol2018first} and federated learning \cite{mcmahan2017communication}, a significant body of work has been devoted to federated meta-learning  \cite{jiang2019improving,chen2018federated,zhengfederated,lin2020collaborative}  and personalized federated learning \cite{fallah2020personalized,wu2020personalized}, under a common theme of fostering edge-edge collaboration. In particular,  federated meta-learning aims to learn a good model initialization (meta-model) across source edge nodes, such that the model of the new task  can be learned via fast adaptation
from the meta-initialization with only a few data samples at the target edge node.

Most of the existing works on  federated meta-learning  focus on the learning from a given set of tasks, each with its  training data, but have not addressed the well-known \emph{catastrophic forgetting} issue   in  continual learning \cite{french1999catastrophic} \cite{parisi2019continual}. Further, the performance of the fast adaptation depends on the similarity among tasks \cite{lin2020collaborative,fallah2020personalized}. As a result,  the meta-model obtained via federated meta-learning may not work well when the tasks on the target node are dissimilar to those at the source nodes. It is also worth noting that most of  the existing federated meta-learning algorithms are  gradient-based, which may suffer from some limitations such as vanishing gradients and sensitivity to poor conditioning \cite{wang2019admm}. It has been observed in practice that these gradient-based algorithms often exhibit slow convergence in training the meta-model, especially on complex tasks \cite{fallah2020personalized,jiang2019improving,chen2018federated},  resulting in low communication efficiency and high computational cost.

To tackle the issues noted above, in this paper, we study continual edge learning via federated meta-learning with regularization. Notably, 
regularization-based methods have been widely used in  continual learning \cite{kirkpatrick2017overcoming,zenke2017continual,schwarz2018progress} and transfer learning \cite{takada2020transfer,si2009bregman}. Inspired by theoretical neuroscience models via imposing constraints on the update of the neural weights \cite{barros2018expectation}, regularization approaches can help to alleviate catastrophic forgetting \cite{parisi2019continual}. Meanwhile, valuable knowledge learned from previous tasks can be extracted as regularization to improve the training speed and the performance of the new task (so-called ``positive forward transfer'' \cite{lopez2017gradient}). Accordingly, we cast the federated meta-learning problem as a regularized stochastic optimization problem, using Bregman divergence \cite{bregman1967relaxation}
to define the regularization. Further, to reduce the computational cost  and  to facilitate collaborative learning, we employ the alternative direction method of multipliers (ADMM) technique to decompose the problem into a set of subproblems that can be solved in parallel across edge nodes and the platform. In particular, by ``decoupling the regularizer'' from the computation at local edge nodes,  it suffices to run the  regularization only in the platform for global aggregation. Observe that the conventional ADMM technique requires the exact solutions to a set of (possibly non-convex) subproblems during each iteration, incurring a possibly high computational cost. To overcome this challenge, we develop a variant of the inexact-ADMM algorithm for the regularized federated meta-learning problem, namely ADMM-FedMeta, where we use linear approximation in each subproblem\footnote{As shown in the proof of convergence of ADMM-FedMeta, it is unnecessary to obtain the exact solutions in each iteration, and this is the underlying rationale of the inexact-ADMM.}, as well as Hessian estimation, and then transform it into a quadratic form that can be solved with a closed-form solution, thus achieving computational complexity of $\mathcal{O}(n)$ per round, with $n$ being the model dimension. 



We note that  the error induced by linear approximation and Hessian estimation, complicates the proof of the convergence of the proposed algorithm, and the  existing   results \cite{wang2019global,barber2020convergence,hong2016convergence} cannot be applied directly, simply because   the \emph{sufficient descent} condition of the Lagrangian function is violated. In this paper, we develop a new technical path to resolve this issue and establish the convergence guarantee for the general non-convex case. Further, we rigorously show that our method can mitigate the catastrophic forgetting and alleviate the performance degradation due to the dissimilarity between the source nodes and the target node.
Besides, different from the previous approaches \cite{lin2020collaborative,fallah2020personalized}, our algorithm can converge under mild conditions, i.e., without regular \emph{similarity} assumptions on the training nodes. Therefore, it can be applied to unbalanced and heterogeneous local datasets, unleashing the potential in dealing with the inherent challenges in federated learning. 

The main contributions of this work are summarized as follows:
\begin{itemize}
\item Aiming to facilitate fast and continual edge learning, we propose a  platform-aided federated-meta learning architecture where edge nodes join forces to learn a meta-model with the knowledge transfer from previous tasks. We cast the edge learning problem as a regularized optimization problem, in which the transferred knowledge is in the form of regularization using Bregman divergence. We  devise an inexact-ADMM based algorithm, called ADMM-FedMeta, where the ADMM technique is employed to decompose the problem into a set of subproblems that can be solved in parallel across edge nodes and the platform, and also it suffices to run the  regularization only in the platform for global aggregation. Further,  a variant of the inexact-ADMM method is devised  where 
the subproblems are `solved' via linear approximation as well as Hessian estimation to reduce the computational cost of per round to $\mathcal{O}(n)$, achieving lower computational complexity compared to most of the existing methods.

\item We carry out a comprehensive  analysis of the proposed algorithm for the general non-convex case, where we establish the convergence and characterize the  performance of  fast adaptation using local samples at the target node. We also quantify the forgetting effect of model knowledge transferred from previous tasks for a special case. Besides, we show that ADMM-FedMeta can mitigate performance degradation incurred by the dissimilarity between the source nodes and the target node.

\item We evaluate the performance of the proposed algorithm on different models and benchmark datasets. Our extensive experimental results showcase that ADMM-FedMeta  outperforms existing state-of-the-art approaches, in terms of convergence speed, adaptation performance, and the capability of learning without forgetting, especially with small sample sizes.

\end{itemize}

\section{Related work}
In this section, we briefly review the related work in the following three areas.

\textbf{Meta-Learning.} Meta-learning has emerged as a promising solution for few-shot learning. 
Ravi {\etal} \cite{ravi2016optimization} propose an LSTM-based meta-learning model to learn an optimization algorithm for training neural networks.
Different from \cite{ravi2016optimization},
a gradient-based Model Agnostic Meta-Learning (MAML) algorithm is proposed in \cite{finn2017model},
which aims at learning a model initialization, based on which using a few gradient descent updates can achieve satisfactory performance on a new task. 
To reduce the computational complexity, Nichol {\etal} \cite{nichol2018first} introduce a first-order meta-learning algorithm called Reptile, which does not require the computation of the second-order derivatives. Multiple follow-up works extend MAML from different perspectives, e.g., \cite{raghu2019rapid,collins2020distribution,song2019maml,finn2018probabilistic}. 
Along a different line,
Fallah {\etal} \cite{fallah2020convergence} establish the convergence of one-step MAML for non-convex loss functions and then proposes a Hessian-free MAML to reduce the computational cost with theoretical guarantees. The convergence for multi-step MAML is studied in \cite{ji2020multi}. Wang {\etal} \cite{wang2020global} further characterize the gap between the stationary point and the global optimum of MAML in a general non-convex setting. 

\textbf{Federated Meta-Learning.} Very recently, the integration of federated learning and MAML has garnered much attention. Chen {\etal} \cite{chen2018federated} propose a federated meta-learning framework called FedMeta based on FedAvg \cite{mcmahan2017communication} and MAML-type algorithms, which improves the performance and convergence speed of FedAvg. Jiang {\etal} \cite{jiang2019improving} analyze the connections between FedAvg and MAML, and proposes a federated meta-learning algorithm called personalized FedAvg. 
Lin {\etal} \cite{lin2020collaborative} analyze the convergence properties and computational complexity of federated meta-learning for a strongly convex setting. Another recent work \cite{fallah2020personalized} proposes a federated meta-learning algorithm called Per-FedAvg and provides the convergence guarantee for the general non-convex setting. However, these studies focus on collaborative learning on a given set of tasks without exploring the valuable knowledge transfer from the previous tasks \cite{parisi2019continual}. 

\textbf{ADMM.} 
A number of existing works \cite{wang2014convergence,hong2016convergence,magnusson2015convergence,wang2019global} analyze the convergence of ADMM for the case where the solution to each subproblem is computed exactly. Wang {\etal} \cite{wang2018convergence} extend the ADMM method from two-block to multi-block form. Besides, there are also a few works \cite{mukkamala2020convex,jiang2019structured,lanza2017nonconvex,barber2020convergence} studying the performance of ADMM in an inexact and non-convex setting, by linearizing the subproblems that are difficult to solve exactly. It is worth noting that linear approximation is insufficient for the meta-learning problem which generally requires higher-order information.

\section{Continual Edge Learning via Federated Meta-Learning With Regularization }

We consider a  platform-aided federated meta-learning architecture for edge learning  (as illustrated in Figure \ref{figure:architecture}), where a set $\mathcal{I}$ of source edge nodes joint force to learn a meta-model, aided by the valuable knowledge learned  from previous tasks in the cloud. Specifically, the  knowledge transfer is in the form of  regularization  using Bregman divergence on the prior model.

\begin{figure}
\centering
\includegraphics[width=0.975\columnwidth]{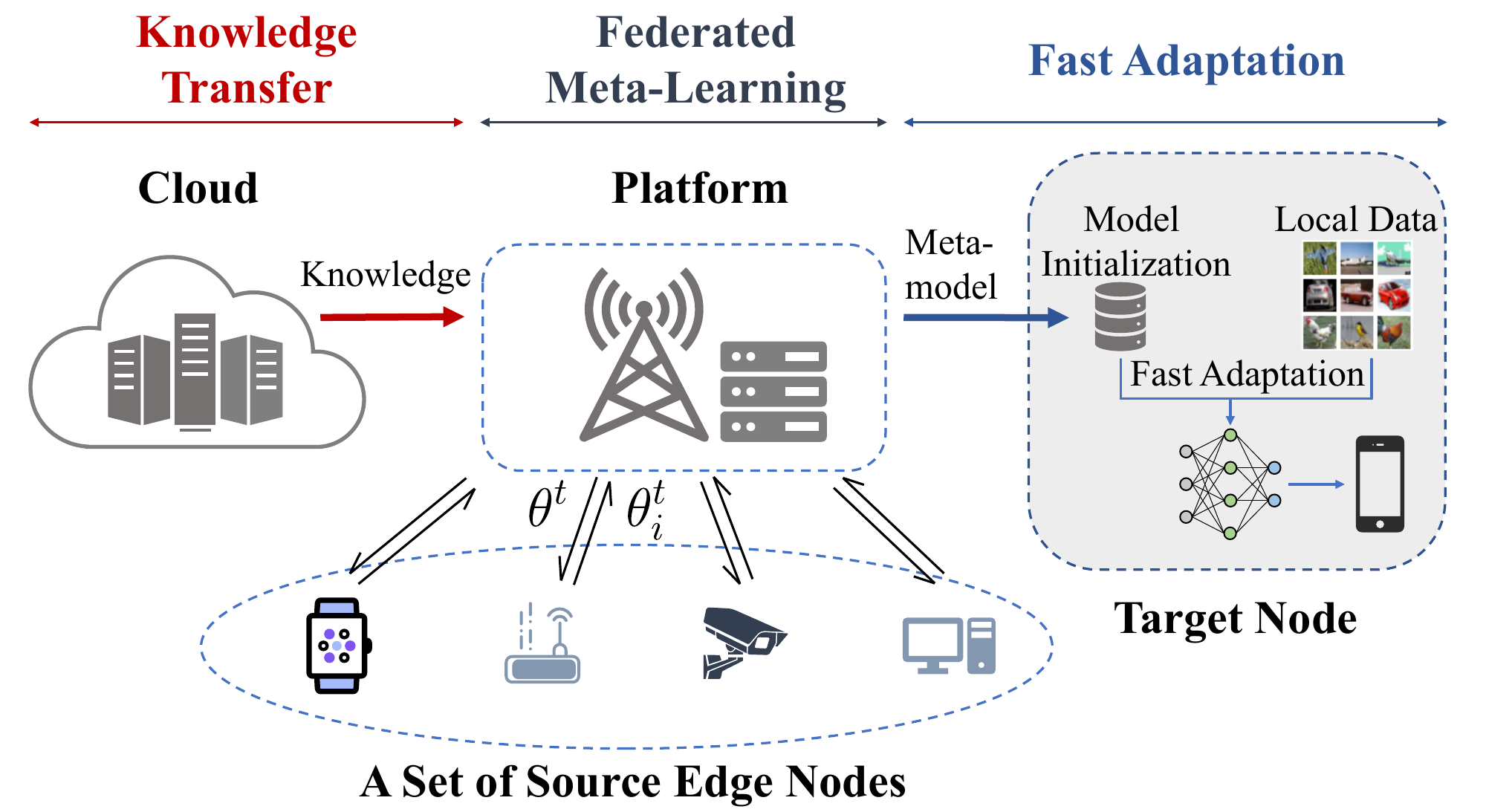} 
\vspace{-6pt}
\caption{Illustration of the platform-aided federated meta-learning architecture with knowledge transfer.}
\vspace{-12pt}
\label{figure:architecture}
\end{figure} 


\subsection{Problem Formulation}

For ease of exposition, we consider a general supervised learning setting where each edge node $i\in\mathcal{I}\cup\{m\}$ has a labeled dataset $\mathcal{D}_i=\big\{(\mathbf{x}^j_i,\mathbf{y}^j_i)\big\}^{D_i}_{j=1}$ with total $D_i$ samples. Here $(\mathbf{x}^j_i,\mathbf{y}^j_i)\in\mathcal{X}_i\times\mathcal{Y}_i$ is a sample point with input $\mathbf{x}^j_i$ and true label $\mathbf{y}^j_i$, and follows an unknown underlying distribution $P_i$. For a model parameter $\phi_i\in\mathbb{R}^n$, the empirical loss function for a dataset $\mathcal{D}_i$ is defined as $L_i(\phi_i,\mathcal{D}_i)\triangleq(1/D_i)\sum^{D_i}_{j=1}l_i\big(\phi_i,(\mathbf{x}^j_i,\mathbf{y}^j_i)\big)$, where $l_i(\cdot,\cdot)$ is a general differentiable non-convex loss function. 

Motivated by the recent success of regularization approaches in transfer learning and continual learning \cite{parisi2019continual}, we  use regularization for extracting  the valuable knowledge from the prior model to facilitate  fast edge training and to alleviate catastrophic forgetting.  More specially, for a model parameter $\theta\in\mathbb{R}^n$, we denote $\theta_p\in\mathbb{R}^n$ as the prior model parameter, and  use the Bregman divergence $D_h(\theta,\theta_p)$ \cite{bregman1967relaxation} as the regularization, given by:
\begin{align}
    D_h(\theta,\theta_p)\triangleq h(\theta)-h(\theta_p)-\langle\nabla h\big(\theta_p\big),\theta-\theta_p\rangle,
\end{align}
for some continuously-differentiable strictly convex function $h(\cdot)$. It is worth noting that Bregman divergence is a dissimilarity measure between two objects (e.g., vectors, matrices, distributions, etc.). It encompasses a rich class of divergence metrics, including squared Euclidean distance, squared Mahalanobis distance, Kullback-Leibler (KL) divergence, and Itakura-Saito (IS) distance, which are widely used in machine learning to encode the dissimilarity from different perspectives, and is particularly useful for the regularization approaches \cite{si2009bregman,yu2013kl,jung2016less}. Note that while for ease of exposition, in this paper we consider  the regularizer  on the model parameters using  Bregman divergence,  the same methodology can be applied to  generalize the  regularization to be in terms of  Bregman divergence between two functions of $\theta$ and $\theta_p$, respectively (see Assumption \ref{Lsmooth}).  

Following the same line as in MAML \cite{finn2017model}, we divide the dataset $\mathcal{D}_i$ for the source edge node $i\in\mathcal{I}$ into two disjoint sets, i.e., the support set $\mathcal{D}^{s}_i$ and the query set $\mathcal{D}^{q}_i$.
Based on the prior model, we can formulate the federated meta-learning with knowledge transfer among the source edge nodes as the following regularized optimization problem:
\begin{align}
    \label{prim_prob}
        \min_\theta&\quad\sum_{i\in\mathcal{I}}w_i L_i\big(\phi_i(\theta),\mathcal{D}^{q}_i\big)+\lambda D_h(\theta,\theta_p)\\
        \label{eq:phi}
        \text{s.t.}&\quad
        \begin{array}{lr}
        \phi_i(\theta)=\theta-\alpha\nabla L_i(\theta,\mathcal{D}^{s}_i),~i\in\mathcal{I}
        \end{array}
\end{align}
where $w_i\triangleq D_i/\sum_{i\in\mathcal{I}}D_i$, 
$\alpha$ is the learning rate, and $\lambda$ is a penalty parameter that can be used to balance the trade-off between the loss and the regularization. In this formulation, we aim to find a good meta-model such that slight updating, i.e., one-step gradient descent, results in substantial performance improvement for any task across the edge nodes. That is, the source edge nodes collaboratively learn how to learn fast with a few data samples. Further, by penalizing changes in the model via regularization, the learned model from  \eqref{prim_prob}-\eqref{eq:phi} is confined to stay `close' to the prior model for enabling collaborative edge learning without forgetting prior knowledge, thus the learned meta-model can widely adapt to different types of tasks.


In the fast adaptation stage, the platform transfers the learned meta-model $\theta$ to the target node (denoted by $m$) after solving the regularized federated meta-learning problem  \eqref{prim_prob}-\eqref{eq:phi}. Based on $\theta$, the target node $m$ can use its local data set $\mathcal{D}^{s}_m$ to quickly compute a new model $\phi_m$ by performing one-step stochastic gradient descent, i.e.,
\begin{align}
    \label{fast_adaptation}
    \phi_m=\theta-\alpha\nabla L_m(\theta,\mathcal{D}^{s}_m).
\end{align}
Note that the target node also can execute a few steps of stochastic gradient descent updates for better performance when needed.



\subsection{An Inexact-ADMM Based Algorithm for Regularized Federated Meta-Learning}

As alluded to earlier, general gradient-based federated meta-learning approaches cannot handle the regularized optimization problem \eqref{prim_prob}-\eqref{eq:phi} well. To address this problem, we propose an inexact-ADMM based federated meta-learning algorithm (ADMM-FedMeta) to solve \eqref{prim_prob}-\eqref{eq:phi}. 

Observe that the federated meta-learning problem \eqref{prim_prob}-\eqref{eq:phi} is equivalent to the following constrained optimization problem:
\begin{align}
    \label{equi_prob}
    \begin{split}
        \min_{\{\theta_i\},\theta}&\quad\sum_{i\in\mathcal{I}}w_i L_i\big(\phi_i(\theta_i),\mathcal{D}^{q}_i\big)+\lambda D_h(\theta,\theta_p)\\
        \text{s.t.}\,&\quad
        \begin{array}{lr}
            \theta_i - \theta = 0,~i\in\mathcal{I}\\
        \end{array}
    \end{split}
\end{align}
where $\phi_i(\theta_i)$ is given by (\ref{eq:phi}). To solve \eqref{equi_prob}, we form the augmented Lagrangian function as follows:
\begin{align}
    \label{eq:lagrangian}
    \nonumber
    \mathcal{L}\big(\{\theta_i,y_i\},\theta\big)\triangleq&\sum_{i\in\mathcal{I}}\Big(w_i L_i\big(\phi_i(\theta_i),\mathcal{D}^{q}_i\big)+\langle y _i,\theta_i-\theta\rangle\\
    &+\frac{\rho_i}{2}\Vert \theta_i-\theta\Vert^2\Big)+\lambda D_h(\theta,\theta_p),
\end{align}
where $y_i\in\mathbb{R}^n$ is a dual variable and $\rho_i>0$ is a penalty parameter for each $i\in\mathcal{I}$. 

When the classical ADMM method is applied \cite{boyd2011distributed}, the variables $\theta_i$, $\theta$ and $y_i$ are updated alternatively in solving \eqref{equi_prob} as follows: 
\begin{align}
    \label{eq:ori_ADMM}
    \begin{cases}
    \theta^{t+1}=\argmin_{\theta}~\mathcal{L}\big(\{\theta^t_i,y^t_i\},\theta\big),\\
    \theta^{t+1}_i=\argmin_{\theta_i}~\mathcal{L}_i\big(\theta_i,y^t_i,\theta^{t+1}\big),\\
    y^{t+1}_i=y^t_i+\rho_i(\theta^{t+1}_i-\theta^{t+1}),\\
    \end{cases}
\end{align}
where $\mathcal{L}_i\big(\theta_i,y_i,\theta\big)\triangleq w_i L_i\big(\phi_i(\theta_i),\mathcal{D}^{q}_i\big)+\langle y _i,\theta_i-\theta\rangle+\frac{\rho_i}{2}\Vert \theta_i-\theta\Vert^2$. The conventional ADMM decomposes the problem \eqref{equi_prob} into a set of subproblems that can be solved in parallel, while computing $D_h(\theta,\theta_p)$ and $L_i\big(\phi_i(\theta_i),\mathcal{D}^{q}_i\big)$ separately. Based on that, to fully take advantage of the combined computation power of the local edge nodes and the platform, we provide the following alternating updating strategy: 1) updating $\theta$ at the platform and 2) updating $\{\theta_i,y_i\}$ at the source edge nodes in a distributed manner. Particularly, in this way, the computation corresponding to the regularizer can be decoupled from the edge nodes to the the platform. However, attaining the exact solution to each subproblem is computationally costly, especially with a complex DNN model. {\em To tackle this challenge, we devise the inexact-ADMM based federated meta-learning (ADMM-FedMeta) below.}

Specifically, in communication round $t=0$, the platform initializes $\theta^0$ and sends it to all edge nodes. Each node  $i\in\mathcal{I}$ initializes $y^{-1}_i$ locally.
\begin{itemize}
 \item\textbf{Local update of $\bm{\{\theta_i,y_i\}}$.} After receiving $\theta^t$ from the platform at communication round $t$, each edge node $i\in\mathcal{I}$ would do the following updates:
    
    (1) \emph{Update node-specific model $\phi_i$.} Based on the dataset $\mathcal{D}^{s}_i$, $\phi^t_i$ is updated as:
    \begin{align}
        \label{eq:update_phi}
        \phi^t_i=\theta^t-\alpha\nabla L_i(\theta^t,\mathcal{D}^{s}_{i}).
    \end{align}
    
    (2) \emph{Update local parameter $\theta_i$.} Based on \eqref{eq:ori_ADMM}, given the meta-model $\theta^t$ and local dual variable $y_i^{t-1}$ from last communication round, the local parameter $\theta_i$ should be updated as:
    \begin{align}\label{eq:update_thetai}
        \theta^t_i=&\mathop{\arg\min}_{\theta_i}\bigg\{w_i L_i\big(\phi_i(\theta_i),\mathcal{D}^{q}_i\big)+ \langle y^{t-1}_i,\theta_i-\theta^t\rangle \nonumber\\
      &+\frac{\rho_i}{2} \Vert\theta_i-\theta^t\Vert^2\bigg\}.
    \end{align}
    To simplify the computation, we use linear approximation (i.e., first-order Taylor expansion) around $\theta^t$ to relax this subproblem, i.e.,
    \begin{align}
    \label{eq:linear_approx}
    \nonumber
    \theta^t_i=&\mathop{\arg\min}_{\theta_i}\bigg\{w_i L_i(\phi^t_i,\mathcal{D}^{q}_i)\\
    \nonumber
    &+  \big\langle w_i\big(I-\alpha\nabla^2 L_i(\theta^t,\mathcal{D}^{s}_i)\big)\nabla  L_i(\phi^t_i,\mathcal{D}^{q}_i)\\
    &+y^{t-1}_i,\theta_i-\theta^t\big\rangle+\frac{\rho_i}{2} \Vert \theta_i-\theta^t\Vert^2\bigg\},
    \end{align}
    where $\phi^t_i$ is from \eqref{eq:update_phi}. Nevertheless, \eqref{eq:linear_approx} is still insufficient since the computational complexity of the Hessian-gradient product $\nabla^2 L_i(\theta^t,\mathcal{D}^{s}_i)\nabla L_i\big(\phi^t_i,\mathcal{D}^{q}_i\big)$ is $\mathcal{O}(n^2)$. To further reduce the computational cost, as in \cite{fallah2020convergence,ji2020multi}, we replace the Hessian-gradient product by a first-order estimator, i.e.,
    \begin{align}
    \label{estimatehessian}
    g^t_i\triangleq \frac{\nabla L_i\big(\theta^t+\delta_{i,t} r^t_i,\mathcal{D}^{s}_i\big)-\nabla L_i\big(\theta^t-\delta_{i,t} r^t_i,\mathcal{D}^{s}_i\big)}{2\delta_{i,t}},
    \end{align}
    where $r^t_i\triangleq\nabla L_i\big(\phi^t_i,\mathcal{D}^{q}_i\big)$ and $\delta_{i,t}>0$ is the degree of freedom capturing the estimation accuracy. In a nutshell, the local parameter $\theta_i^t$ is updated as follows:
    \begin{align}
    \label{eq:update_thetai_inex}
    \theta^t_i=\theta^t-\frac{y^{t-1}_i+w_i\big(\nabla L_i(\phi^t_i,\mathcal{D}^q_{i})-\alpha g^t_i\big)}{\rho_i},
    \end{align}
    where \eqref{eq:update_thetai_inex} is derived by the optimality of (\ref{eq:linear_approx}) after replacing $\nabla^2 L_i(\theta^t,\mathcal{D}^{s}_i)\nabla L_i\big(\phi^t_i,\mathcal{D}^{q}_i\big)$ with $g^t_i$. 
    \label{eq:update_thetai_inex_sp}
    
    (3) \emph{Update local dual variable $y_i$.} Based on $\theta^t$ and the updated local parameter $\theta_i^t$, the auxiliary dual variable $y_i^t$ is next updated according to:
    \begin{align}
        \label{eq:update_y}
        y^t_i=& y^{t-1}_i + \rho_i(\theta^t_i - \theta^t).
    \end{align}
    
\item\textbf{Global Aggregation towards Meta-Model $\bm{\theta}$.} Each edge node $i\in\mathcal{I}$ sends the updated local parameters $\theta^t_i$ and $y^t_i$ to the platform. With the prior model $\theta_p$ transferred from the cloud, the platform performs a global update of the model initialization $\theta$ based on:
    \begin{align}\label{eq:update_theta}
        \theta^{t+1}=\frac{\sum_{i\in\mathcal{I}}(y^t_i+\rho_i\theta^t_i)-\lambda\nabla_{\theta^t}D_h(\theta^t,\theta_p)}{\sum_{i\in\mathcal{I}}\rho_i},
    \end{align}
where \eqref{eq:update_theta} is derived from the optimality of the linearized $\mathcal{L}\big(\{\theta^t_i,y^t_i\},\theta\big)$ around $\theta^t$ similar to \eqref{eq:update_thetai_inex}. Then, the platform sends $\theta^{t+1}$ back to all edge nodes for the next communication round.

\item\textbf{Fast Adaptation.} After the training phase, the platform transfers the learned meta-model $\theta^T$ to the target node $m$. Based on $\theta^T$, the target node performs one or a few steps of stochastic gradient descent on its own dataset to obtain a new model.

\end{itemize}
The details of ADMM-FedMeta are summarized in Algorithm \ref{alg}. Note that due to linearizing all decomposed subproblems and estimating Hessian by its first-order estimation, we enable the computation complexity of ADMM-FedMeta to be $\mathcal{O}(n)$ per round, which maintains the lowest among all existing federated meta-learning approaches. 
\begin{algorithm}[t]
	\caption{Inexact-ADMM Based Meta-Learning Algorithm (ADMM-FedMeta)}
	\label{alg}
	\LinesNumbered
	\KwIn{$\theta_p$, $\alpha$, $\lambda$, $\rho_i$, $\mathcal{D}^{s}_i$, $\mathcal{D}^q_i$ for $i\in\mathcal{I}$}
	\KwOut{Finial meta-model $\theta$}
	Each edge node $i\in\mathcal{I}$ initializes $y^{-1}_i$\;
	Platform initializes $\theta^0$ and sends it to all edge nodes\;
	\For{$t=0$ \KwTo $T$}{
	    \For{$i=1$ \KwTo $I$}{
	        Compute $\phi^t_i\leftarrow\theta^t-\alpha\nabla L_i(\theta^t,\mathcal{D}^{s}_i)$\;
	        Compute $\theta^t_i$ by (\ref{eq:update_thetai_inex})\;
	        Compute $ y^t_i\leftarrow y^{t-1}_i + \rho_i(\theta^t_i - \theta^t)$\;
	        Send $\theta^t_i$ and $ y^t_i$ back to the platform\;
	        }
	    Platform updates $\theta^{t+1}$ by (\ref{eq:update_theta}) and sends it to all edge nodes $i\in\mathcal{I}$\;
		}
		Platform transfers $\theta^T$ to target node for fast adaptation\;
\end{algorithm}

\section{Performance Analysis}

In this section, we analyze the performance of ADMM-FedMeta. First, we study the convergence properties and characterize the communication complexity for ADMM-FedMeta. Then, we quantify the forgetting effect to previous tasks of the meta-model and analyze the fast adaptation performance at the target edge node.

\subsection{Convergence Analysis}

For convenience, we denote the objective function of \eqref{prim_prob} as $F(\theta)$:
\begin{align}
    \label{F_definition}
    F(\theta)\triangleq \sum_{i\in\mathcal{I}}w_i L_i\big(\phi_i(\theta),\mathcal{D}^{q}_i\big)+\lambda D_h(\theta,\theta_p),
\end{align}
where $\phi_i(\theta)=\theta-\alpha\nabla L_i(\theta,\mathcal{D}^{s}_i)$. Next, we characterize the convergence and communication complexity of the proposed algorithm for finding a first-order stationary point of function $F(\theta)$. Formally, the definition of an $\epsilon$-approximate first-order stationary point is given as follows.
\begin{definition}[$\epsilon$-FOSP]
A solution $\theta\in\mathbb{R}^n$ is called an $\epsilon$-approximate first-order stationary point ($\epsilon$-FOSP) of (\ref{prim_prob}), if $   \Vert\nabla F(\theta)\Vert\le\epsilon,~\text{for}~\epsilon>0$.
\end{definition}
The above definition implies that if a solution $\theta$ obtained by an algorithm is a $\epsilon$-FOSP, then the gradient norm of the objective function is bounded above by $\epsilon$.

Note that the first-order estimator of Hessian introduced in the subproblem \eqref{eq:update_thetai_inex} inevitably complicates the convergence analysis of ADMM-FedMeta, making the existing analysis methods of ADMM \cite{barber2020convergence} not suitable here. To establish the convergence of ADMM-FedMeta, we impose the following standard assumptions in the literature.
\begin{assumption}
    \label{lowerbounded}
    $F(\theta)$ is lower-bounded, i.e.,  $F(\theta)>-\infty$, for all $\theta\in\mathbb{R}^n$.
\end{assumption}
\begin{assumption}[Smoothness and Bounded Gradient]
\label{Lsmooth}
For each $i\in\mathcal{I}\cup\{m\}$, any $\mathcal{D}^{s}_i$, and $\theta_p\in\mathbb{R}^n$, both $L_i(\cdot,\mathcal{D}^{s}_i)$ and $D_h(\cdot,\theta_p)$ are twice continuously differentiable and smooth, i.e., for any $x,y\in\mathbb{R}^n$, there exist constants $\mu_i>0$ and $\mu_r>0$ such that:
\begin{align}
    \label{Lsmoothinequality}
    &\Vert \nabla L_i(x,\mathcal{D}^{s}_i)-\nabla L_i(y,\mathcal{D}^{s}_i)\Vert\le \mu_i\Vert x-y\Vert,\\
    &\Vert \nabla_x D_h(x,\theta_p)-\nabla_y D_h(y,\theta_p)\Vert\le \mu_r\Vert x-y\Vert.
\end{align}
Besides, the gradient norms of $L_i(\cdot,\mathcal{D}^{s}_i)$ is bounded by a positive constant $\beta_i>0$, i.e., for any $x\in\mathbb{R}^n$, the following holds:
\begin{align}
    \label{gradientbound}
    &\Vert\nabla L_i(x,\mathcal{D}^{s}_i)\Vert\leq \beta_i.
\end{align}
\end{assumption}
\begin{assumption}[Lipschitz Continuous Hessian]
\label{HessianLipschitz}
For any $i\in\mathcal{I}$ and any $\mathcal{D}^{s}_i$, the Hessian of $L_i(\cdot,\mathcal{D}^{s}_i)$ is $\zeta_i$-Lipschitz continuous, i.e., for any $x,y\in\mathbb{R}^n$, we have:
\begin{align}
    \label{HessianLipschitzinequality}
    \Vert \nabla^2 L_i(x,\mathcal{D}^{s}_i)-\nabla^2 L_i(y,\mathcal{D}^{s}_i)\Vert\le \zeta_i\Vert x-y\Vert.
\end{align}
\end{assumption}
\begin{assumption}[Bounded Variance]
\label{bounded_var}
For any $i\in\mathcal{I}\cup\{m\}$ and $\theta\in\mathbb{R}^n$, the stochastic gradient $\nabla l_i\big(\theta,(\mathbf{x},\mathbf{y})\big)$ and Hessian $\nabla^2 l_i\big(\theta,(\mathbf{x},\mathbf{y})\big)$ with respect to data point $(\mathbf{x},\mathbf{y})\in\mathcal{X}_i\times\mathcal{Y}_i$ have bounded variances, i.e.,
\begin{align}
    &\mathbb{E}_{(\mathbf{x},\mathbf{y})\sim P_i}\big\{\Vert\nabla l_i\big(\theta,(\mathbf{x},\mathbf{y})\big)-\nabla L_i(\theta)\Vert^2\big\}\le\big(\sigma^g_i\big)^2,\\
    &\mathbb{E}_{(\mathbf{x},\mathbf{y})\sim P_i}\big\{\Vert\nabla^2 l_i\big(\theta,(\mathbf{x},\mathbf{y})\big)-\nabla^2 L_i(\theta)\Vert^2\big\}\le\big(\sigma^h_i\big)^2,
\end{align}
for some positive constants $\sigma^g_i>0$ and $\sigma^h_i>0$.
\end{assumption}

Assumptions \ref{lowerbounded}-\ref{bounded_var} are standard in the state-of-the-art studies on the analysis of federated learning algorithms \cite{lin2020collaborative,fallah2020personalized,zhang2020fedpd}. In particular, (\ref{gradientbound}) is critical for analyzing the convergence as it enables characterizing the estimation error of the Hessian. Assumption \ref{HessianLipschitz} implies the high-order smoothness of $L_i(\cdot,\mathcal{D}^{s}_i)$ for dealing with the second-order information in the update steps of Algorithm \ref{alg}.  Furthermore, Assumption \ref{bounded_var} provides the upper bounds of the variances of the gradient and Hessian estimations.

To quantify the convergence behavior of ADMM-FedMeta, we  first study the properties of the objective function $F(\theta)$. Denote $F_i(\theta)\triangleq L_i\big(\phi_i(\theta),\mathcal{D}^q_i\big)$. Based on Assumptions \ref{Lsmooth} and \ref{HessianLipschitz}, we have the following result about the smoothness of $F_i$ as in the standard analysis of federated meta-learning approaches.
\begin{restatable}{lemma}{fsmoo}
\label{fsmoothlemma}
Given Assumptions \ref{Lsmooth} and \ref{HessianLipschitz}, for each $i\in\mathcal{I}\cup\{m\}$, $F_i$ is proper and $\nu_i$-smooth, i.e., 
\begin{align}
    \label{fsmooth}
    \Vert \nabla F_i(x)-\nabla F_i(y)\Vert\le \nu_i\Vert x-y\Vert,~\forall x,y\in\mathbb{R}^n,
\end{align}
where $\nu_i$ is defined as follows:
\begin{align}
\label{lipschtz_nu}
    \nu_i\triangleq (1+\alpha\mu_i)(1+\mu_i)\mu_i+\alpha\beta_i\zeta_i.
\end{align}
\end{restatable}
\begin{proof}
The proof is standard. The detailed proof is provided in Appendix \ref{proof:smooth} of the technical report \cite{yue2020inexactadmm}.
\end{proof}
Next, we impose the assumptions on the hyper-parameters.
\begin{assumption}
    \label{parameterassumption}
    For all $i\in\mathcal{I}$, $\rho_i$ is large enough such that:
    \begin{align}
        \label{parameter_eq1}
        &\frac{\rho_i}{2}-4w_i\nu_i>0,\\
        \label{parameter_eq2}
        &\frac{\rho_i}{2}-2w^2_i\nu^2_i\bigg(\frac{4w_i\nu_i}{\rho^2_i}+\frac{1}{\rho_i}\bigg)-\frac{\lambda\mu_r}{2I}>0,\\
        \label{parameter_eq3}
        &\rho_i-3\nu_i>0,
    \end{align}
    where $\nu_i$ is a smooth scalar defined in (\ref{lipschtz_nu}). Besides, for all $i\in\mathcal{I}$, the additional degree of freedom parameter $\{\delta_{i,t}\}$ for the approximation  of Hessian-gradient products  is chosen to be a monotonically non-increasing positive sequence and satisfies $\sum^\infty_{t=1}\delta_{i,t}<\infty$.
\end{assumption}
We impose Assumption \ref{parameterassumption} on the penalty parameter $\rho_i$ and the degree of freedom parameter $\delta_{i,t}$. Intuitively, (\ref{parameter_eq1})-(\ref{parameter_eq3}) imply that a large $\rho_i$ is required to balance the error caused by the linear approximation and Hessian estimation in \eqref{eq:update_thetai_inex}.

Based on Lemma \ref{fsmoothlemma}, we are ready to establish the convergence and characterize the communication complexity for Algorithm \ref{alg}. 

\begin{restatable}[Convergence and Communication Complexity]{theorem}{conv}
\label{convergence}
Under Assumptions \ref{lowerbounded}-\ref{parameterassumption}, we have the following results based on  Algorithm \ref{alg}:
\begin{enumerate}[(i)]
    \item $\{\theta^t\}$ has at least one limit point and each limit point $\theta^*$ is a stationary solution of (\ref{prim_prob}), i.e., $\Vert \nabla F(\theta^*)\Vert=0$.
    \item Algorithm \ref{alg} finds an $\epsilon$-FOSP of Problem \eqref{prim_prob}-\eqref{eq:phi} after at most $\mathcal{O}(1/\epsilon^2)$ communication rounds.
\end{enumerate}
\end{restatable}
\begin{proof}
For part (i), we first characterize the successive difference of the augmented Lagrangian function. Based on that, we next show $\lim_{t\rightarrow\infty}\Vert\theta^{t+1}_i-\theta^t_i\Vert=0$ and $\lim_{t\rightarrow\infty}\Vert\theta^{t+1}-\theta^t\Vert=0$. Lastly, we bound $\Vert \nabla F(\theta^t)\Vert$ via $\Vert\theta^{t+1}_i-\theta^t_i\Vert$ and $\Vert\theta^{t+1}-\theta^t\Vert$. We prove (ii) via dividing the sum of the Lagrangian successive difference into two finite parts. The detailed proof is presented in Appendix \ref{proof:convergence} of the technical report \cite{yue2020inexactadmm}.
\end{proof}
Theorem \ref{convergence} indicates that Algorithm \ref{alg} always converges to a stationary point of (\ref{prim_prob}). Besides, to find an $\epsilon$-FOSP of Problem \eqref{prim_prob}-\eqref{eq:phi}, Algorithm \ref{alg} requires $\mathcal{O}(1/\epsilon^2)$ communication rounds between edge nodes and the platform. It is worth noting that in contrast to the previous methods \cite{lin2020collaborative,fallah2020personalized,fallah2020convergence}, ADMM-FedMeta can converge under mild conditions, i.e., not depending on the similarity assumptions (i.e., Assumption \ref{similarity}) across different edge nodes. This implies that Algorithm \ref{alg} can be applied to unbalanced and heterogeneous local datasets, revealing the potential in dealing with the inherent challenges in federated learning.

To characterize the impact of local data samples on the expected performance on the source nodes, we provide the following corollary.
\begin{corollary}
\label{coro:data_impact}
Given Assumptions \ref{lowerbounded}-\ref{parameterassumption}, the $\epsilon$-FOSP solution $\theta_{\epsilon}$ found by Algorithm \ref{alg} satisfies that:
\begin{align}
    \nonumber
    &\mathbb{E}\Big\{\Big\Vert \sum_{i\in\mathcal{I}}w_i L_i\big(\theta_{\epsilon}-\alpha\nabla L_i(\theta_{\epsilon})\big)+\lambda D_h(\theta_{\epsilon},\theta_p)\Big\Vert\Big\}\\
    &\le\epsilon+\sum_{i\in\mathcal{I}}w_i\sigma^g_i\Bigg(\frac{\alpha\mu_i}{\sqrt{D^{s}_i}}+\frac{1}{\sqrt{D^{q}_i}}\Bigg),
\end{align}
where $w_i=1/I$ and $L_i(\cdot)$ is the expected loss denoted by:
\begin{align}
    \label{eq:expected_loss}
    L_i(\theta)\triangleq\mathbb{E}_{(\mathbf{x},\mathbf{y})\sim P_i}\big\{l_i\big(\theta,(\mathbf{x},\mathbf{y})\big)\big\}.
\end{align}
\end{corollary}
\begin{proof}
The detailed proof is presented in Appendix \ref{proof:coro} of the technical report \cite{yue2020inexactadmm}.
\end{proof}
Corollary \ref{coro:data_impact} implies that despite $\epsilon$-FOSP of the deterministic loss \eqref{prim_prob} can be obtained within $\mathcal{O}(1/\epsilon^2)$ rounds, performance degradation may happen due to small sample size.

\subsection{Performance of Rapid Adaptation at Target Node}
While the task similarity assumption is not required to guarantee the convergence of Algorithm \ref{alg}, we impose such an assumption to study the fast adaptation performance at the target node $m$.
\begin{assumption}[Task Similarity]
\label{similarity}
There exist positive constants $\psi^g_i>0$ and $\psi^h_i>0$ such that for any $i\in\mathcal{I}$ and $\theta\in\mathbb{R}^n$, the following holds:
\begin{align}
    &\Vert\nabla L_m(\theta)-\nabla L_i(\theta)\Vert\le \psi^g_i,\\
    &\Vert\nabla^2 L_m(\theta)-\nabla^2 L_i(\theta)\Vert\le \psi^h_i,
\end{align}
where $L_i(\cdot)$ is defined in \eqref{eq:expected_loss}, and the same applies to $L_m(\cdot)$.
\end{assumption}
Assumption \ref{similarity} indicates that the variations of the gradients between the loss of source edge nodes and the target edge node are bounded by some constants, which capture the similarity of the tasks corresponding to non-IID data and holds for many practical loss functions \cite{zhang2020fedpd}, such as logistic regression and hyperbolic tangent functions. In particular, $\psi^g_i$ and $\psi^h_i$ can be roughly seen as a distance between data distributions $P_m$ and $P_i$ \cite{fallah2020convergence}.

Next, we present the following result about the performance of rapid adaptation.
\begin{restatable}[Fast Adaptation Performance]{theorem}{perf}
\label{performance}
Suppose that Assumptions \ref{lowerbounded}-\ref{similarity} hold. For any $\epsilon>0$, the $\epsilon$-FOSP solution $\theta_{\epsilon}$ obtained by Algorithm \ref{alg} satisfies that:
\begin{align}
    \label{performance_ineq}
    \nonumber
    \mathbb{E}\big\{\Vert \nabla F_m(\theta_{\epsilon})\Vert\big\}\le&\epsilon+\alpha\beta_m\sum_{i\in\mathcal{I}}w_i\psi^h_i+(\alpha\mu+1)^2\sum_{i\in\mathcal{I}}w_i\psi^g_i\\
    \nonumber
    &+\alpha\mu(\alpha\mu+1)\sum_{i\in\mathcal{I}}w_i\sigma^g_i\Bigg(\frac{1}{\sqrt{D^{q}_m}}+\frac{1}{\sqrt{D^{q}_i}}\Bigg)\\
    \nonumber
    &+(\alpha\mu+1)\sum_{i\in\mathcal{I}}w_i\sigma^g_i\Bigg(\frac{1}{\sqrt{D^{s}_m}}+\frac{1}{\sqrt{D^{s}_i}}\Bigg)\\
    &+\alpha\beta_m\sum_{i\in\mathcal{I}}w_i\sigma^h_i\Bigg(\frac{1}{\sqrt{D^{s}_m}}+\frac{1}{\sqrt{D^{s}_i}}\Bigg),
\end{align}
where $F_m(\theta)\triangleq L_m\big(\theta-\alpha\nabla L_m(\theta,\mathcal{D}^{s}_m),\mathcal{D}^q_m\big)+\lambda D_h(\theta,\theta_p)$ for any $\mathcal{D}^{s}_m$ and $\mathcal{D}^q_m$ with respect to distribution $P_m$, and $\mu=\max_{i\in\mathcal{I}}\{\mu_i\}$.
\end{restatable}
\begin{proof}
The detailed proof is presented in Appendix \ref{proof:performance} of the technical report \cite{yue2020inexactadmm}.
\end{proof}
Theorem \ref{performance} sheds light on the performance of fast adaptation with the previous knowledge, which depends on the size of datasets, the variance of stochastic gradient and Hessian, and the similarity between the target node and source nodes. In particular, if $D^{s}_i=\mathcal{O}(\epsilon^{-2})$ and $D^{q}_i=\mathcal{O}(\epsilon^{-2})$ for $i\in\mathcal{I}\cup\{m\}$, then an $\mathcal{O}\big(\epsilon+\sum_{i\in\mathcal{I}}w_i(\psi^h_i+\psi^g_i)\big)$-FOSP can be obtained at the target node. However, it is clear that the larger the dataset of source nodes dissimilar to the target node is, the worse the rapid adaptation performs. In the next subsection, we will show these issues can be alleviated via regularization with a good prior model.

\subsection{Forgetting effect on Prior Knowledge}
In this section, we quantify the forgetting effect of the previous task of Algorithm \ref{alg} in a special case, where the regularizer is squared Euclidean distance, i.e., $D_h(\theta,\theta_p)=\Vert\theta-\theta_p\Vert^2$. To do so, we first derive an upper bound of $D_h(\theta,\theta_p)$ via the following lemma. 
\begin{lemma}
\label{lem:forgetting}
Given Assumptions \ref{lowerbounded}-\ref{parameterassumption}, for any $\epsilon>0$, the $\epsilon$-FOSP solution $\theta_{\epsilon}$ obtained by Algorithm \ref{alg} satisfies that:
\begin{align}
    \label{eq:forgetting_general_ex}
    \mathbb{E}\big\{D_h(\theta_{\epsilon},\theta_p)\big\}\le\frac{1}{\lambda}\left(\epsilon+\sum_{i\in\mathcal{I}}w_i\left(\beta_i+\frac{\alpha\mu_i\sigma^g_i}{\sqrt{D^s_i}}+\frac{\sigma^g_i}{\sqrt{D^q_i}}\right)\right)\Vert\theta_{\epsilon}-\theta_p\Vert.
\end{align}
Particularly, suppose $D_h(\theta,\theta_p)$ is strongly convex with respect to $\theta$, i.e., there exists $M>0$ such that for $x,y\in\mathbb{R}^n$, the following holds:
\begin{align}
    \label{eq:strongly_convex_primal_ex}
    \left\langle \nabla D_h(x,\theta_p)-\nabla D_h(y,\theta_p),x-y\right\rangle\ge M\Vert x-y\Vert^2.
\end{align}
Then \eqref{eq:forgetting_general_ex} can be written as:
\begin{align}
    \label{eq:forgetting_general_strong_ex}
    \mathbb{E}\big\{D_h(\theta_{\epsilon},\theta_p)\big\}\le\frac{1}{M\lambda^2}\left(\epsilon+\sum_{i\in\mathcal{I}}w_i\left(\beta_i+\frac{\alpha\mu_i\sigma^g_i}{\sqrt{D^s_i}}+\frac{\sigma^g_i}{\sqrt{D^q_i}}\right)\right)^2.
\end{align}
\end{lemma}
\begin{proof}
The detailed proof is presented in Appendix \ref{proof:forgetting} of the technical report \cite{yue2020inexactadmm}.
\end{proof}

For a current model parameter $\theta$, we define the \emph{forgetting cost} of $\theta$ on the previous task $p$ as $L_p(\theta)$ \cite{krishnan2020meta}, where $L_p(\cdot)$ is the expected loss over the data distribution of task $p$ (defined by \eqref{eq:expected_loss}). Based on Lemma \ref{lem:forgetting}, we next characterize the forgetting cost of the $\epsilon$-FOSP solution.
\begin{theorem}
\label{thm:forgetting}
Suppose that $L_p(\cdot)$ is $\mu_p$-smooth, $\Vert\nabla L_p(\theta_p)\Vert\le\epsilon_p$ for some $\epsilon_p>0$, and $D_h(\theta,\theta_p)=\Vert\theta-\theta_p\Vert^2$. Under Assumptions \ref{lowerbounded}-\ref{parameterassumption}, for the $\epsilon$-FOSP solution $\theta_{\epsilon}$, we have the following result:
\begin{align}
    \label{eq:forgeting_eucliean}
    \Vert\nabla L_p(\theta_{\epsilon})\Vert\le\epsilon_p+\frac{\nu_p}{2\lambda}\left(\epsilon+\sum_{i\in\mathcal{I}}w_i\left(\beta_i+\frac{\alpha\mu_i\sigma^g_i}{\sqrt{D^s_i}}+\frac{\sigma^g_i}{\sqrt{D^q_i}}\right)\right).
\end{align}
\end{theorem}
\begin{proof}
The result can be directly obtained by Lemma \ref{lem:forgetting}.
\end{proof}
Based on Theorem \ref{thm:forgetting} and corollary \ref{coro:data_impact}, it is clear that by selecting a suitable $\lambda$, the regularizer enables the meta-model to learn on the current task while maintaining good performance on the previous task. On the other hand, combined with Theorem \ref{performance}, Theorem \ref{thm:forgetting} also implies that due to the independence of \eqref{eq:forgeting_eucliean} and similarity conditions, a good prior reference model (e.g., with a relatively small $\Vert \nabla L_m(\theta_p)\Vert$) can effectively alleviate the significant performance degradation caused by the dissimilarity between the source nodes and the target nodes. 


\section{Experimental Results}

\begin{table*}[htbp]
\centering
\begin{tabular}{ccccc}
\toprule
Dataset                                                                   & \# local updates & FedAvg             & Per-FedAvg         & ADMM-FedMeta       \\ \midrule
\multirow{3}{*}{\begin{tabular}[c]{@{}l@{}}Fashion-\\ MNIST\end{tabular}} & 1     & $83.99\%\pm3.90\%$ & $87.55\%\pm2.42\%$ & \bf{95.69\%$\pm$0.37\%} \\
                                                                          & 5     & $88.86\%\pm1.57\%$ & $89.65\%\pm3.26\%$ & N/A                  \\
                                                                          & 10    & $85.29\%\pm1.93\%$ & $90.95\%\pm2.71\%$ & N/A                  \\ \midrule
\multirow{3}{*}{CIFAR-10}                                                 & 1     & $41.97\%\pm1.33\%$ & $60.53\%\pm1.12\%$ & \bf{74.61\%$\pm$2.19\%} \\
                                                                          & 5     & $56.58\%\pm2.27\%$ & $65.93\%\pm9.97\%$ & N/A                  \\
                                                                          & 10    & $56.58\%\pm1.15\%$ & $67.43\%\pm0.99\%$ & N/A                  \\ \midrule
\multirow{3}{*}{CIFAR-100}                                                & 1     & $42.35\%\pm1.55\%$ & $48.19\%\pm2.18\%$ & \bf{63.56\%$\pm$0.87\%} \\
                                                                          & 5     & $50.00\%\pm1.09\%$ & $49.56\%\pm1.09\%$ & N/A                  \\
                                                                          & 10    & $49.97\%\pm1.04\%$ & $48.73\%\pm1.23\%$ & N/A                  \\ \bottomrule
\end{tabular}
\caption{Comparison of the accuracy on target nodes of different algorithms.}
\vspace{-14pt}
\label{table:comparison_confidence}
\end{table*}

\begin{figure*}[t]
\centering
\subfigure{
\includegraphics[width=0.25\linewidth]{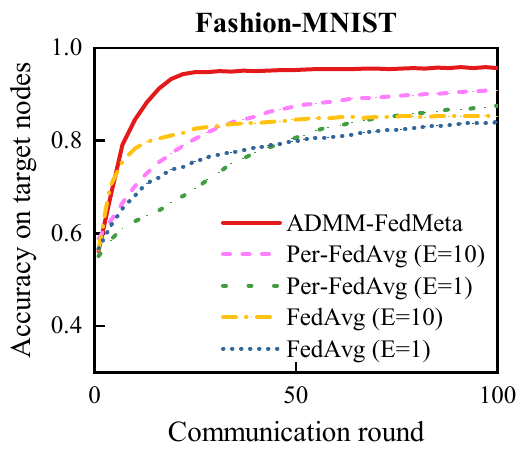}
}
\subfigure{
\includegraphics[width=0.25\linewidth]{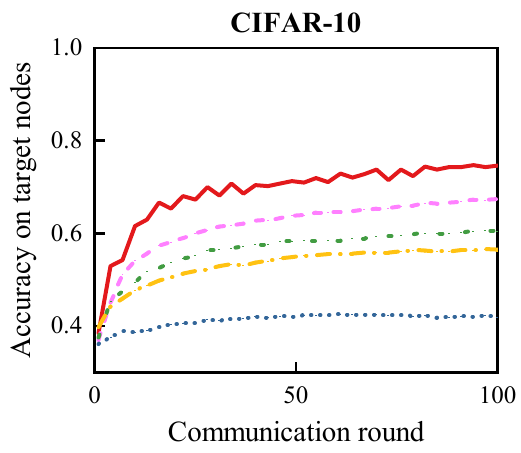}
}
\subfigure{
\includegraphics[width=0.25\linewidth]{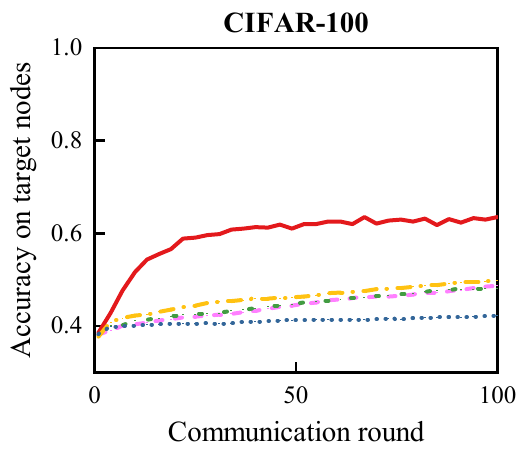}
}
\vspace{-12pt}
\caption{Comparison of the convergence speed of different algorithms.}
\vspace{-6pt}
\label{figure:convergence}
\end{figure*}

In this section, we evaluate the experimental performance of ADMM-FedMeta on different datasets and models. In particular, our experimental studies are designed to evaluate the performance of the proposed ADMM-FedMeta algorithm in challenging edge learning settings where edge nodes have limited data samples. Specifically, we assume that each source node has only tens of data samples during the training stage and that in the testing phase, each target node has only 10-20 data samples. Clearly, edge learning in these settings is highly nontrivial, particularly for sophisticated datasets (e.g., CIFAR-100).
\label{subsection:experimental_setup}

\textbf{Datasets and models.} We evaluate the performance of ADMM-FedMeta on three widely-used benchmarks, including Fashion-MNIST \cite{xiao2017fashion}, CIFAR-10 \cite{krizhevsky2009learning}, and CIFAR-100 \cite{krizhevsky2009learning}. Specifically, the data is distributed among $I$ edge nodes as follows: 1) Each node has samples from only two random classes \cite{lin2020collaborative}; 2) the number of samples per node follows a discrete uniform distribution, i.e., $D_i\sim U(a,b)$ for $i\in\mathcal{I}$. Here we set $a=20$, $b=40$, $I=50$ for Fashion-MNIST and CIFAR-10, and $I=100$ for CIFAR-100. We randomly select 80\% and 20\% nodes as the source nodes and the target nodes respectively. For each node, we divide the local dataset into a support set and a query set (i.e., $\mathcal{D}^{s}_i$ and $\mathcal{D}^q_i$), each with 50\% of the local data. We set the meta-step stepsize as $\alpha=0.01$, the penalty parameters $\rho=0.3$ for Fashion-MNIST, and $\rho=0.7$ for CIFAR-10 and CIFAR-100, where $\rho_i=\rho$ for all $i\in\mathcal{I}$. We set the regularizer as squared $\ell_2$-norm, and the degree of freedom parameter $\delta_{i,t}=1/(10t+100)$ with $t=1,2,\dots,100$ for $i\in\mathcal{I}$. For Fashion-MNIST, we use a convolutional neural network (CNN) with max-pooling operation and Exponential Linear Unit (ELU) activation function, which contains two convolutional layers with sizes 32 and 64 followed by a fully connected layer and softmax. The strides are set as 1 for convolution operation and 2 for pooling operation. For CIFAR-10 and CIFAT-100, we use a CNN containing three convolutional layers with sizes 32, 64, and 128, and a CNN containing four convolutional layers with sizes 32, 64, 128, and 256, respectively, while keeping all the other setups the same as that in Fashion-MNIST. 

\textbf{Implementation.}
We implement the code in TensorFlow Version 1.14 on a server with two Intel$^\circledR$ Xeon$^\circledR$ Golden 5120 CPUs and one Nvidia$^\circledR$ Tesla-V100 32G GPU. Please refer to \url{https://github.com/XinJiang1994/HFmaml} for full details.

\textbf{Baselines.} We consider two existing baseline algorithms, i.e., FedAvg \cite{mcmahan2017communication} and Per-FedAvg \cite{fallah2020personalized} with one or multiple local update steps. For the sake of fair comparison,  we test different hyper-parameters of Per-FedAvg from $\{0.001,0.005,0.01,0.05,0.1\}$ (i.e.,$\beta$ in \cite[Algorithm 1]{fallah2020personalized}), and select the best for the experiments, i.e., 0.005 for Fashion-MNIST and CIFAR-10, and 0.001 for CIFAR-100. 

To demonstrate the impact of the knowledge transfer and the inexact-ADMM based methods respectively, we first remove the regularization term (i.e., letting $\lambda=0$) and compare the convergence speed and adaptation performance between ADMM-FedMeta and the baselines. Then, we conduct the experiment using a prior model for regularization and show the performance improvement in terms of convergence, adaptation, and forgetting effect.


\textbf{Performance and computational efficiency.} To be fair,  we set $\lambda=0$ to remove the benefit of using the regularization in ADMM-FedMeta. We repeat the experiments 10 times, then show the comparison of the accuracy along with 95\% confidence intervals in Table \ref{table:comparison_confidence}. We have the following observations. (1) ADMM-FedMeta substantially outperforms Per-FedAvg and FedAvg, especially on sophisticated datasets. Specifically, ADMM-FedMeta achieves 7.7\% over FedAvg and 5.2\% over Per-FedAvg on Fashion-MNIST, 31.87\% over  FedAvg and 10.65\% overPer-FedAvg on CIFAR-10, and 27.12\% over FedAvg and 28.25\% over Per-FedAvg on CIFAR-100. Note that the computation costs of each local update are $\mathcal{O}(n)$, $\mathcal{O}(n^2)$, and $\mathcal{O}(n)$ for FedAvg, Per-FedAvg, and ADMM-FedMeta, respectively. This performance improvement clearly indicates that ADMM-FedMeta is more computationally efficient on non-convex loss and heterogeneous datasets with a small number of data samples. (2) It should be noted that the gaps between ADMM-FedMeta and the two baselines on CIFAR-10 are smaller when the number of local updates increases. The underlying rationale is that with more local update steps, the number of the overall iterations in the two baseline algorithms actually increases, thus resulting in a better model. However, a large number of local update steps would lead to high computational cost and may cause failure to convergence \cite[Theorem 4.5]{fallah2020personalized}. 

\textbf{Convergence speed and hyper-parameter.} As shown in Figure \ref{figure:convergence}, ADMM-FedMeta converges significantly faster than the existing methods, often requiring tens of rounds to obtain a high-quality meta-model, which indicates ADMM-FedMeta can achieve a great communication efficiency. Besides, Figure \ref{figure:convergence} also suggests that despite the sample size is small, edge nodes can obtain a satisfactory model via federated meta-learning with only one-step stochastic gradient descent. Further, we investigate the impact of the hyper-parameter $\rho$ on the convergence of ADMM-FedMeta (we let $\rho_i$ be the same across different nodes). It can be seen from Figure \ref{figure:rho} that ADMM-FedMeta has a relatively faster convergence speed with a smaller $\rho$ in terms of the training loss. In particular, a small change of $\rho$ does not greatly affect the convergence properties of the algorithm, which implies that ADMM-FedMeta is robust to the hyper-parameters.

\begin{figure*}[t]
\centering
\subfigure{
\includegraphics[width=0.25\linewidth]{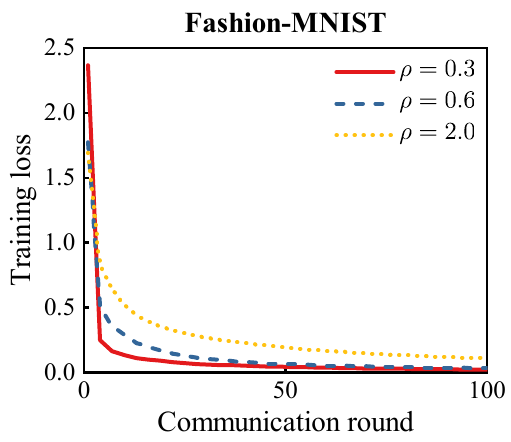}
}
\subfigure{
\includegraphics[width=0.25\linewidth]{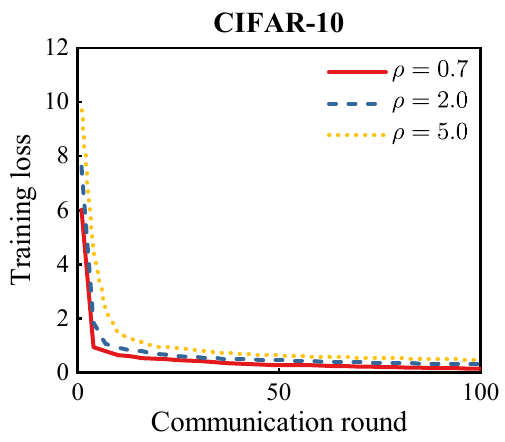}
}
\subfigure{
\includegraphics[width=0.25\linewidth]{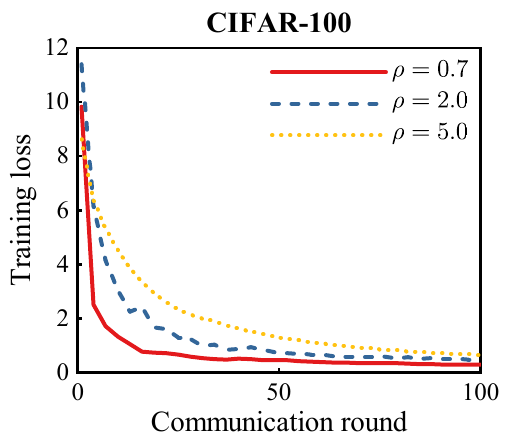}
}
\vspace{-12pt}
\caption{Impact of penalty parameter $\rho$.}
\vspace{0pt}
\label{figure:rho}
\end{figure*}

\begin{table}[htpb]
\begin{tabular}{ccccccc}
\toprule
Dataset                                                                   & Task & \begin{tabular}[c]{@{}c@{}}Prior\\ Model\end{tabular} & FedAvg  & \begin{tabular}[c]{@{}c@{}}Per-\\ FedAvg\end{tabular} & \begin{tabular}[c]{@{}c@{}}ADMM-\\ FedMeta\end{tabular} &  \\ \midrule
\multirow{2}{*}{\begin{tabular}[c]{@{}c@{}}Fashion-\\ MNIST\end{tabular}} & Prior  & 95.63\%                                               & 41.27\% & 49.60\%                                               & \textbf{92.86\%}                                        &  \\
                                                                          & New  & 49.21\%                                                   & 94.05\% & 94.84\%                                               & \textbf{94.04\%}                                        &  \\ \midrule
\multirow{2}{*}{CIFAR-10}                                                 & Prior  & 75.74\%                                               & 41.08\% & 42.08\%                                               & \textbf{62.38\%}                                        &  \\
                                                                          & New  & 17.33\%                                                   & 55.45\% & 47.03\%                                               & \textbf{71.29\%}                                        &  \\ \midrule
\multirow{2}{*}{CIFAR-100}                                                & Prior  & 66.27\%                                               & 35.32\% & 37.62\%                                               & \textbf{59.52\%}                                        &  \\
                                                                          & New  & 45.63\%                                                   & 40.48\% & 57.92\%                                               & \textbf{63.10\%}                                        &  \\ \bottomrule
\end{tabular}
\caption{Comparison of the accuracy on prior and current tasks.}
\vspace{-16pt}
\label{table:forgetting}
\end{table}

\textbf{Forgetting effect to the previous task.} To demonstrate the forgetting effect to the prior task on different algorithms, we pre-train a model $\theta_p$ with satisfactory performance on the data of the first five classes as a prior task, and then use $\theta_p$ as the initialization to train the meta-models by different algorithms on the data of the last five classes as a new task. Similarly, over CIFAR-100, we use the first fifty classes and the last fifty classes as the prior and new tasks respectively. After that, We test the adaptation performance of the meta-models on the prior and new tasks to show the forgetting effect. We set $\lambda=0.5$ for Fashion-MNIST, $\lambda=1$ for CIFAR-10 and CIFAR-100, and use squared Euclidean distance as the regularizer. As illustrated in Table \ref{table:forgetting}, the existing methods suffer from the catastrophic forgetting on the previous task due to a lack of mechanisms to extract the knowledge from the prior model. Clearly, ADMM-FedMeta can effectively mitigate this issue via a regularization with the prior model, while also maintaining a satisfactory performance on the new task.

\begin{figure}[t]
\centering
\includegraphics[width=0.975\columnwidth]{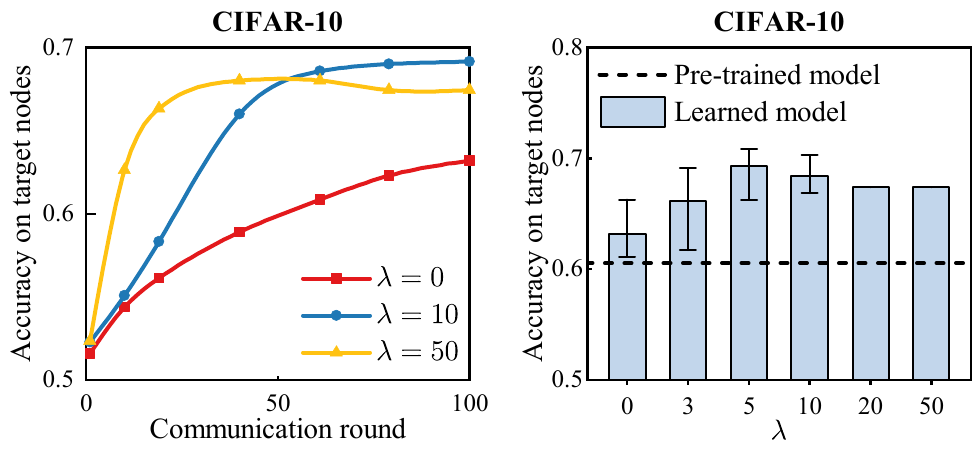} 
\vspace{-10pt}
\caption{Impact of $\theta_p$ and $\lambda$.}
\vspace{-16pt}
\label{figure:lambda}
\end{figure}

\textbf{Impact of prior knowledge.} To quantify the impact of the knowledge transfer on the convergence of model training and the adaptation performance of target nodes, we pre-train a prior model as $\theta_p$ using  images of 3-10 classes on CIFAR-10. Then, we train the meta-model $\theta$ on source nodes with images of 1-8 classes and test the accuracy on the target nodes with all 1-10 classes. In this way, the pre-trained model can be considered as containing valuable knowledge of the current task. As shown in Figure \ref{figure:lambda}, with the useful knowledge transferred from the previous task to the edge nodes, faster convergence, and higher adaptation performance are achieved by appropriately selecting the trade-off parameter $\lambda$. In other words, the regularization with useful knowledge transfer can help get a high-quality meta-model and achieve faster edge learning. 

\section{Conclusion}
In this paper, we presented an  inexact-ADMM based federated meta-learning approach for  fast and continual edge learning. More specifically, we first proposed a platform-aided federated meta-learning architecture enabling edge nodes to collaboratively learn a meta-model with the knowledge transfer of previous tasks. We cast the federated meta-learning problem as a regularized optimization problem, where the previous knowledge is extracted as regularization. Then, we devised an ADMM based algorithm, namely ADMM-FedMeta, in which the original problem is decomposed into many subproblems which can be solved in parallel across edge nodes and the platform. Further, we developed a variant of the inexact-ADMM method to reduce the computational cost per round to $\mathcal{O}(n)$ via employing  linear approximation as well as Hessian estimation. We provided a comprehensive analysis and empirical results to demonstrate the effectiveness and efficiency of ADMM-FedMeta. The advantages of the proposed algorithm are summarized as follows: First, it can decouple the regularizer from edge nodes to the platform, which helps to  alleviate the local computational cost while exploiting the resources between local devices and the server effectively. Besides, by inexact-ADMM technique, we further reduce the computational complexity during local update and global aggregation. We show that ADMM-FedMeta can converge under  mild conditions, particularly,  with weak \emph{task similarity} assumptions. Lastly, empirical results show that ADMM-FedMeta converges  faster than existing benchmark algorithms.

There are a number of interesting questions and directions for future work. First, it is of interest to incorporate the experience replay method and parameter isolation approaches into the ADMM-FedMeta to further mitigate the catastrophic forgetting. Secondly, despite ADMM-FedMeta can be directly applied to reinforcement learning   with policy gradient, it may  lead to poor sample efficiency. It remains largely open to develop efficient collaborative reinforcement learning  for edge learning. Moreover, our experimental results indicate that even without the regularization term, in practice ADMM-FedMeta can still converge  faster than the existing gradient-based methods, especially on small sample sizes. It is intriguing to get a more deep understanding of this phenomenon. 

\section{Acknowledgments}

This research was supported in part by NSF under Grants CNS-2003081 and CPS-1739344, National Key R\&D Program of China under Grant No. 2019YFA0706403, National Natural Science Foundation of China under Grants No. 62072472, 61702562 and U19A2067,  Natural Science Foundation of Hunan Province, China under Grant No. 2020JJ2050, 111 Project under Grant No. B18059, the Young Elite Scientists Sponsorship Program by CAST under Grant No. 2018QNRC001, the Young Talents Plan of Hunan Province of China under Grant No. 2019RS2001, and also financially supported by China 
Scholarship Council (CSC).


\bibliographystyle{ACM-Reference-Format}
\bibliography{reference}


\onecolumn
\section*{Appendix}
\appendix


\section{Proof of Lemma \ref{fsmoothlemma}}
\label{proof:smooth}
The proof is standard. For simplicity, we denote $F_i(\theta)\triangleq L_i\big(\theta-\alpha\nabla L_i(\theta,\mathcal{D}^{s}_i),\mathcal{D}^q_i\big)$. 
Recall that $\nabla F_i(x)=\big(I_n-\alpha\nabla^2 L_i(x,\mathcal{D}^{s}_i)\big)\nabla L_i\big(x-\alpha\nabla L_i(x,\mathcal{D}^{s}_i),\mathcal{D}^q_i\big)$, and we have
\begin{align}
\label{fsmooth_eq1}
\big\Vert\nabla F_i(x)-\nabla F_i(y)\big\Vert\le&\big\Vert \nabla L_i\big(x-\alpha\nabla L_i(x,\mathcal{D}^{s}_i),\mathcal{D}^q_i\big)-\nabla L_i\big(y-\alpha\nabla L_i(y,\mathcal{D}^{s}_i),\mathcal{D}^q_i\big) \big\Vert\\
\nonumber
&+\alpha \big\Vert\nabla^2 L_i(x,\mathcal{D}^{s}_i)\nabla L_i\big(x-\alpha\nabla L_i(x,\mathcal{D}^{s}_i),\mathcal{D}^q_i\big)\\
\label{fsmooth_eq2}
&-\nabla^2 L_i(y,\mathcal{D}^{s}_i)\nabla L_i\big(y-\alpha\nabla L_i(y,\mathcal{D}^{s}_i),\mathcal{D}^q_i\big)\big\Vert,
\end{align}
and
\begin{align}
    \label{Hessianbound}
    -\mu_i I_n\preceq \nabla^2 L_i(x,\mathcal{D}^{s}_i) \preceq \mu_i I_n,~\forall x\in\mathbb{R}^n.
\end{align}
To prove (\ref{fsmooth}), we need to bound (\ref{fsmooth_eq1}) and (\ref{fsmooth_eq2}). For (\ref{fsmooth_eq1}), based on Assumption \ref{Lsmooth}, we have 
\begin{align}
    \label{fsmooth_eq3}
    \nonumber
    &\big\Vert \nabla L_i\big(x-\alpha\nabla L_i(x,\mathcal{D}^{s}_i),\mathcal{D}^q_i\big)-\nabla L_i\big(y-\alpha\nabla L_i(y,\mathcal{D}^{s}_i),\mathcal{D}^q_i\big) \big\Vert\\
    \nonumber
    \le&\mu_i\big\Vert x-y-\alpha\big(\nabla L_i(x,\mathcal{D}^{s}_i)-\nabla L_i(y,\mathcal{D}^{s}_i)\big)\big\Vert\\
    \nonumber
    \le&\mu_i\big( \Vert x-y\Vert + \alpha\Vert \nabla L_i(x,\mathcal{D}^{s}_i)-\nabla L_i(y,\mathcal{D}^{s}_i)\Vert\big)\\
    \le& (1+\alpha\mu_i)\mu_i\Vert x-y\Vert.
\end{align}
To bound (\ref{fsmooth_eq2}), it can be shown that
\begin{align}
    \label{fsmooth_eq4}
    \nonumber
    &\big\Vert\nabla^2 L_i(x,\mathcal{D}^{s}_i)\nabla L_i\big(x-\alpha\nabla L_i(x,\mathcal{D}^{s}_i),\mathcal{D}^q_i\big)-\nabla^2 L_i(y,\mathcal{D}^{s}_i)\nabla L_i\big(y-\alpha\nabla L_i(y,\mathcal{D}^{s}_i),\mathcal{D}^q_i\big)\big\Vert\\
    \nonumber
    =&\big\Vert\nabla^2 L_i(x,\mathcal{D}^{s}_i)\nabla L_i\big(x-\alpha\nabla L_i(x,\mathcal{D}^{s}_i),\mathcal{D}^q_i\big)-\nabla^2 L_i(x,\mathcal{D}^{s}_i)\nabla L_i\big(y-\alpha\nabla L_i(y,\mathcal{D}^{s}_i),\mathcal{D}^q_i\big)\\
    \nonumber
    &+\nabla^2 L_i(x,\mathcal{D}^{s}_i)\nabla L_i\big(y-\alpha\nabla L_i(y,\mathcal{D}^{s}_i),\mathcal{D}^q_i\big)-\nabla^2 L_i(y,\mathcal{D}^{s}_i)\nabla L_i\big(y-\alpha\nabla L_i(y,\mathcal{D}^{s}_i),\mathcal{D}^q_i\big)\big\Vert\\
    \nonumber
    \le& \big\Vert \nabla^2 L_i(x,\mathcal{D}^{s}_i)\big\Vert\cdot\big\Vert \nabla L_i\big(x-\alpha\nabla L_i(x,\mathcal{D}^{s}_i),\mathcal{D}^q_i\big)-\nabla L_i\big(y-\alpha\nabla L_i(y,\mathcal{D}^{s}_i),\mathcal{D}^q_i\big)\big\Vert 
    \\
    \nonumber
    &+\big\Vert \nabla^2 L_i(x,\mathcal{D}^{s}_i)-\nabla^2 L_i(y,\mathcal{D}^{s}_i)\big\Vert\cdot\big\Vert\nabla L_i\big(y-\alpha\nabla L_i(y,\mathcal{D}^{s}_i),\mathcal{D}^q_i\big)\big\Vert\\
    \nonumber
    \le& \Big((1+\alpha\mu_i)\mu^2_i+\zeta_i\big\Vert\nabla L_i\big(y-\alpha\nabla L_i(y,\mathcal{D}^{s}_i),\mathcal{D}^q_i\big)\big\Vert\Big)\Vert x-y\Vert\\
    \le& \big((1+\alpha\mu_i)\mu^2_i+\zeta_i\beta_i\big)\Vert x-y\Vert,
\end{align}
where (\ref{fsmooth_eq4}) follows from (\ref{gradientbound}), (\ref{Hessianbound}), (\ref{fsmooth_eq3}) and Assumption \ref{HessianLipschitz}. 
Combining (\ref{fsmooth_eq3}) and (\ref{fsmooth_eq4}) yields the result.

\section{Proof of Lemma \ref{dualbound}}
Based on Lemma \ref{fsmoothlemma}, we next prove the Lemmas \ref{dualbound}-\ref{lagrangian_lowerbound} for the convergence analysis. In the following lemma, we first bound the variations of $y^t_i$ via the variations of $\theta^t$. For simplicity, denote 
\begin{align}
    f_i(\theta)\triangleq L_i(\theta,\mathcal{D}^{s}_i),~
    f^q_i(\phi_i)\triangleq L_i(\phi_i,\mathcal{D}^q_i),~
    F_i(\theta)\triangleq L_i\big(\theta-\alpha\nabla L_i(\theta,\mathcal{D}^{s}_i),\mathcal{D}^q_i\big).
\end{align}
\begin{lemma}
\label{dualbound}
Suppose that Assumption \ref{lowerbounded}-\ref{HessianLipschitz} are satisfied. Then, the following holds true
\begin{align}
    \label{dualbound_ineq_1}
    \Vert y^{t+1}_i- y^t_i\Vert\le w_i\nu_i\Vert\theta^{t+1}-\theta^t\Vert + (\delta_{i,t}+\delta_{i,t+1})\alpha w_i\zeta_i\beta^2_i.
\end{align}
\end{lemma}
\begin{proof}
\label{proof_dualbound}
First, define
\begin{align}
    \label{tilde_F}
    \tilde{\nabla}F_i(\theta^{t+1})\triangleq \nabla f^q_i(\phi^{t+1}_i)-\alpha g^{t+1}_i,
\end{align}
where $\phi^{t+1}_i=\theta^{t+1}-\alpha\nabla f_i(\theta^{t+1})$. We have the following observation from (\ref{eq:update_thetai_inex})
\begin{align}
    \label{dualbound_eq1}
    w_i\tilde{\nabla}F_i(\theta^{t+1})+ y^t_i+\rho_i(\theta^{t+1}_i-\theta^{t+1})=0.
\end{align}
Using (\ref{eq:update_y}), we conclude that (\ref{dualbound_eq1}) is equivalent to
\begin{align}
    \label{dualbound_eq2}
    - y^{t+1}_i=w_i\tilde{\nabla}F_i(\theta^{t+1}).
\end{align}
Thus, for all $t\in\mathbb{N}$, the following is true
\begin{align}
    \label{dualbound_eq3}
    \Vert y^{t+1}_i- y^t_i\Vert= w_i\underbrace{\Vert\tilde{\nabla}F_i(\theta^{t+1})- \tilde{\nabla}F_i(\theta^t)\Vert}_{\text{(a)}}.
\end{align}
Using Lemma \ref{fsmoothlemma}, we derive the upper bound of (a) as follows
\begin{align}
    \label{dualbound_eq4}
    \nonumber
    \text{(a)}
    =&\big\Vert\big(\nabla F_i(\theta^{t+1})-\tilde{\nabla}F_i(\theta^{t+1})\big)-\big(\nabla F_i(\theta^{t+1})-\tilde{\nabla}F_i(\theta^t)\big)\big\Vert\\
    \nonumber
    \le&\Vert\nabla F_i(\theta^{t+1})-\tilde{\nabla}F_i(\theta^{t+1})\Vert+\Vert\nabla F_i(\theta^{t+1})-\tilde{\nabla}F_i(\theta^t)\Vert\\
    \nonumber
    \le&\Vert\nabla F_i(\theta^{t+1})-\tilde{\nabla}F_i(\theta^{t+1})\Vert+\Vert\nabla F_i(\theta^{t+1})-\nabla F_i(\theta^t)\Vert+\Vert\nabla F_i(\theta^t)-\tilde{\nabla}F_i(\theta^t)\Vert\\
    \le& \nu_i\Vert\theta^{t+1}-\theta^t\Vert + (\delta_{i,t}+\delta_{i,t+1})\alpha\zeta_i\beta^2_i,
\end{align}
where the last equality uses the following result in  \cite{fallah2020convergence}
\begin{align}
    \big\Vert\nabla^2 f_i(\theta^{t+1})\nabla f^q_i(\phi^{t+1}_i)-g^{t+1}_i\big\Vert\le\delta_{i,t+1}\zeta_i\beta^2_i.
\end{align}
Plugging (\ref{dualbound_eq4}) into (\ref{dualbound_eq3}), we have
\begin{align}
    \Vert y^{t+1}_i- y^t_i\Vert\le&w_i\nu_i\Vert\theta^{t+1}-\theta^t\Vert + (\delta_{i,t}+\delta_{i,t+1})\alpha w_i\zeta_i\beta^2_i,
\end{align}
which completes the proof.
\end{proof}

\section{Proof of Lemma \ref{l_descent}}
To bound the successive difference of the augmented Lagrangian function $\mathcal{L}\big(\{\theta^t_i,y^t_i\},\theta^t\big)$ defined in (\ref{eq:lagrangian}), we first bound the successive difference of $\mathcal{L}_i(\theta_i,\theta^{t+1}, y^t_i)$, which is defined as follows
\begin{align}
    \mathcal{L}_i(\theta_i,\theta^{t+1}, y^t_i)\triangleq& w_i F_i(\theta_i)+\langle y^t _i,\theta_i-\theta^{t+1}\rangle+\frac{\rho_i}{2}\Vert \theta_i-\theta^{t+1}\Vert^2.
\end{align}
Then, we have the following lemma.

\begin{lemma}
\label{l_descent}
Suppose that Assumption \ref{lowerbounded}-\ref{HessianLipschitz} are satisfied. The following holds true
\begin{align}
    \nonumber
    \mathcal{L}_i(\theta^{t+1}_i,\theta^{t+1}, y^t_i)- \mathcal{L}_i(\theta^t_i,\theta^{t+1}, y^t_i)
     \le& -\frac{\rho_i-(3+4w_i)\nu_i}{2}\Vert\theta^{t+1}_i-\theta^t_i\Vert^2+\frac{2(1+w_i)\nu_i}{\rho^2_i}\Vert  y^{t+1}_i- y^t_i\Vert^2\\
     \nonumber
     &+\frac{2\alpha w_i\zeta_i\beta^2_i\delta_{i,t+1}}{\rho_i}\Vert  y^{t+1}_i- y^t_i\Vert+\alpha w_i\zeta_i\beta^2_i\delta_{i,t+1}\Vert \theta^t_i-\theta^{t+1}_i\Vert.
\end{align}
\end{lemma}
\begin{proof}
\label{proof_l_descent}
First, we define $\hat{\mathcal{L}}_i(\theta_i,\theta^{t+1}, y^t_i)$ and $\tilde{\mathcal{L}}_i(\theta_i,\theta^{t+1}, y^t_i)$ as follows
\begin{align}
    \hat{\mathcal{L}}_i(\theta_i,\theta^{t+1}, y^t_i)\triangleq& w_i\big\langle\big(I-\alpha\nabla^2 f_i(\theta^{t+1})\big)\nabla f^q_i(\phi^{t+1}_i),\theta_i-\theta^{t+1}\big\rangle+w_i F_i(\theta^{t+1})+\langle y^t_i,\theta_i-\theta^{t+1}\rangle\\
    &+\frac{\rho_i}{2} \Vert \theta_i-\theta^{t+1}\Vert^2,\\
    \tilde{\mathcal{L}}_i(\theta_i,\theta^{t+1}, y^t_i)\triangleq& w_i\big\langle\nabla f^q_i(\phi^{t+1}_i)-\alpha g^{t+1}_i,\theta_i-\theta^{t+1}\big\rangle+w_i F_i(\theta^{t+1})+\langle y^t_i,\theta_i-\theta^{t+1}\rangle\\
    &+\frac{\rho_i}{2} \Vert \theta_i-\theta^{t+1}\Vert^2,
\end{align}
where $g^{t+1}_i$ and $\phi^{t+1}_i$ are defined in (\ref{estimatehessian}) and \eqref{eq:update_phi}, respectively. 
For each $i\in\mathcal{I}$, using Lemma \ref{fsmoothlemma} yields
\begin{align}
    \label{l_descent_eq1}
    \mathcal{L}_i(\theta_i,\theta^{t+1}, y^t_i)\le\hat{\mathcal{L}}_i(\theta_i,\theta^{t+1}, y^t_i)+\frac{w_i\nu_i}{2} \Vert \theta_i-\theta^{t+1}\Vert^2.
\end{align}
Recall that 
    \begin{align}
        \Vert\nabla^2 f_i(\theta^{t+1})\nabla f^q_i(\phi^{t+1}_i)-g^{t+1}_i\Vert\le\zeta_i\beta^2_i\delta_{i,t+1}.
    \end{align}
Thus, using the Cauchy-Schwarz inequality, we can write
\begin{align}
    \label{l_descent_eq2}
    \hat{\mathcal{L}}_i(\theta_i,\theta^{t+1}, y^t_i)\le& \tilde{\mathcal{L}}_i(\theta_i,\theta^{t+1}, y^t_i)+\alpha w_i\zeta_i\beta^2_i\delta_{i,t+1}\Vert \theta_i-\theta^{t+1}\Vert.
\end{align}
Combining (\ref{l_descent_eq1}) and (\ref{l_descent_eq2}) yields that
\begin{align}
    \label{l_descent_eq3}
    \mathcal{L}_i(\theta_i,\theta^{t+1}, y^t_i)\le& \tilde{\mathcal{L}}_i(\theta_i,\theta^{t+1}, y^t_i)+\frac{w_i\nu_i}{2} \Vert \theta_i-\theta^{t+1}\Vert^2+\alpha w_i\zeta_i\beta^2_i\delta_{i,t+1}\Vert \theta_i-\theta^{t+1}\Vert.
\end{align}
Based on (\ref{dualbound_eq1}) and the strong convexity of $\tilde{\mathcal{L}}_i(\theta_i,\theta^{t+1}, y^t_i)$ with modulus $\rho_i$, we can show that for each $i\in\mathcal{I}$  
\begin{align}
    \label{l_descent_eq4}
    \tilde{\mathcal{L}}_i(\theta^{t+1}_i,\theta^{t+1}, y^t_i)-\tilde{\mathcal{L}}_i(\theta^t_i,\theta^{t+1}, y^t_i)\le-\frac{\rho_i}{2}\Vert\theta^{t+1}_i-\theta^t_i\Vert^2.
\end{align}
It follows that
\begin{align}
    \label{l_descent_eq5}
    \nonumber
    &\tilde{\mathcal{L}}_i(\theta^t_i,\theta^{t+1}, y^t_i)-\mathcal{L}_i(\theta^t_i,\theta^{t+1}, y^t_i)\\
    \nonumber
    =&w_i\big\langle\nabla f^q_i(\phi^{t+1}_i)-\alpha g^{t+1}_i,\theta^t_i-\theta^{t+1}\big\rangle+w_i F_i(\theta^{t+1})-w_i F_i(\theta^t_i)\\
    \nonumber
    =&w_i F_i(\theta^{t+1})-w_i F_i(\theta^t_i)-w_i\langle \nabla F_i(\theta^t_i),\theta^{t+1}-\theta^t_i\rangle-\frac{w_i\nu_i}{2} \Vert \theta^t_i-\theta^{t+1}\Vert^2+\frac{w_i\nu_i}{2} \Vert \theta^t_i-\theta^{t+1}\Vert^2\\
    \nonumber
    &+w_i\big\langle\nabla f^q_i(\phi^{t+1}_i)-\alpha g^{t+1}_i-\nabla F_i(\theta^t_i),\theta^t_i-\theta^{t+1}\big\rangle\\
    \nonumber
    \mathop{\le}^{\text{(a)}}& w_i\big\langle\nabla f^q_i(\phi^{t+1}_i)-\alpha g^{t+1}_i-\nabla F_i(\theta^t_i),\theta^t_i-\theta^{t+1}\big\rangle+\frac{w_i\nu_i}{2}\Vert \theta^t_i-\theta^{t+1}\Vert^2\\
    \nonumber
    \le& w_i\Vert\nabla f^q_i(\phi^{t+1}_i)-\alpha g^{t+1}_i - \nabla F_i(\theta^{t+1})\Vert\cdot\Vert \theta^t_i-\theta^{t+1}\Vert+w_i\Vert\nabla F_i(\theta^{t+1})- \nabla F_i(\theta^t_i)\Vert\cdot\Vert \theta^t_i-\theta^{t+1}\Vert+\frac{w_i\nu_i}{2}\Vert \theta^t_i-\theta^{t+1}\Vert^2\\
    \nonumber
    \le&3/2\cdot w_i\nu_i\Vert \theta^t_i-\theta^{t+1}\Vert^2 + \alpha w_i\zeta_i\beta^2_i\delta_{i,t+1}\Vert \theta^t_i-\theta^{t+1}\Vert\\
    \mathop{\le}^{\text{(b)}}&3w_i\nu_i\big(\Vert \theta^t_i-\theta^{t+1}_i\Vert^2+\Vert \theta^{t+1}_i-\theta^{t+1}\Vert^2\big)+\alpha w_i\zeta_i\beta^2_i\delta_{i,t+1}\big(\Vert \theta^t_i-\theta^{t+1}_i\Vert+\Vert \theta^{t+1}_i-\theta^{t+1}\Vert\big),
\end{align}
where (a) is derived from Lemma \ref{fsmoothlemma} and (b) is based on the following fact 
    \begin{align}
        \Vert x+y\Vert^2\le 2\Vert x\Vert^2+2\Vert y\Vert^2,~x,y\in\mathbb{R}^n.
    \end{align}
Combining (\ref{l_descent_eq3})-(\ref{l_descent_eq5}), we conclude that
\begin{align}
    \label{l_descent_eq6}
    \nonumber
     &\mathcal{L}_i(\theta^{t+1}_i,\theta^{t+1}, y^t_i)- \mathcal{L}_i(\theta^t_i,\theta^{t+1}, y^t_i)\\
     \nonumber
     \le& \tilde{\mathcal{L}}_i(\theta^{t+1}_i,\theta^{t+1}, y^t_i)-\tilde{\mathcal{L}}_i(\theta^t_i,\theta^{t+1}, y^t_i)+\tilde{\mathcal{L}}_i(\theta^t_i,\theta^{t+1}, y^t_i)-\mathcal{L}_i(\theta^t_i,\theta^{t+1}, y^t_i)\\
     \nonumber
     &+\frac{w_i\nu_i}{2} \Vert \theta^t_i-\theta^{t+1}\Vert^2+\alpha w_i\zeta_i\beta^2_i\delta_{i,t+1}\Vert \theta^{t+1}_i-\theta^{t+1}\Vert\\
     \nonumber
     \le& -\frac{\rho_i-8w_i\nu_i}{2}\Vert\theta^{t+1}_i-\theta^t_i\Vert^2+4w_i\nu_i\Vert \theta^{t+1}_i-\theta^{t+1}\Vert^2+2\alpha w_i\zeta_i\beta^2_i\delta_{i,t+1}\Vert \theta^{t+1}_i-\theta^{t+1}\Vert\\
     \nonumber
     &+\alpha w_i\zeta_i\beta^2_i\delta_{i,t+1}\Vert \theta^t_i-\theta^{t+1}_i\Vert\\
     \nonumber
     \mathop{\le}^{\text{(c)}}& -\frac{\rho_i-8w_i\nu_i}{2}\Vert\theta^{t+1}_i-\theta^t_i\Vert^2+\frac{4w_i\nu_i}{\rho^2_i}\Vert  y^{t+1}_i- y^t_i\Vert^2+\frac{2\alpha w_i\zeta_i\beta^2_i\delta_{i,t+1}}{\rho_i}\Vert  y^{t+1}_i- y^t_i\Vert\\
     &+\alpha w_i\zeta_i\beta^2_i\delta_{i,t+1}\Vert \theta^t_i-\theta^{t+1}_i\Vert,
\end{align}
where (c) is derived from (\ref{eq:update_y}). This completes the proof.
\end{proof}

\section{Proof of Lemma \ref{L_descent}}

Based on Lemma \ref{l_descent}, we derive the successive difference of the augmented Lagrangian function $\mathcal{L}\big(\{\theta^t_i,y^t_i\},\theta^t\big)$ in the following lemma. Note that due to the error induced by linear approximation and first-order Hessian estimation, the \emph{sufficient descent} does not hold below. 

\begin{lemma}
\label{L_descent}
Suppose that Assumption \ref{lowerbounded}-\ref{HessianLipschitz} hold. Then the following holds 
\begin{align}
    \label{L_descent_ineq}
    \nonumber
    \mathcal{L}\big(\{\theta^{t+1}_i,y^{t+1}_i\},\theta^{t+1}\big)-\mathcal{L}\big(\{\theta^t_i,y^t_i\},\theta^t\big)
    \le&-\sum_{i\in\mathcal{I}}\Big(a_{i,e}\Vert\theta^{t+1}_i-\theta^t_i\Vert^2+a_{i,p}\Vert\theta^{t+1}-\theta^t\Vert^2\\
    &-b^{t+1}_{i,e}\Vert\theta^{t+1}_i-\theta^t_i\Vert-b^{t+1}_{i,p}\Vert\theta^{t+1}-\theta^t\Vert-c^{t+1}_i\Big),
\end{align}
where $a_{i,e}$ and $a_{i,p}$ are defined in (\ref{parameter_eq1}) and (\ref{parameter_eq2}), respectively. $b^{t+1}_{i,e}$, $b^{t+1}_{i,e}$, and $c^{t+1}_i$ are defined as follows
\begin{gather}
    \label{parameter_be}
    b^{t+1}_{i,e}\triangleq \alpha w_i\zeta_i\beta^2_i\delta_{i,t+1},\\
    \label{parameter_bp}
    b^{t+1}_{i,p}\triangleq \frac{2\alpha\nu_i w^2_i\zeta_i\beta^2_i\delta_{i,t+1}}{\rho_i},\\
    \label{parameter_c}
    c^{t+1}_i\triangleq 2(\delta_{i,t}+\delta_{i,t+1})^2(\alpha w_i\zeta_i\beta^2_i)^2\bigg(\frac{4w_i\nu_i}{\rho^2_i}+\frac{1}{\rho_i}\bigg)+\frac{2(\alpha w_i\zeta_i\beta^2_i)^2\delta_{i,t+1}}{\rho_i}(\delta_{i,t}+\delta_{i,t+1}).
\end{gather}
\end{lemma}
\begin{proof}
\label{proof_L_descent}
Based on the update (\ref{eq:update_y}), we first obtain
\begin{align}
    \label{L_descent_eq1}
    \nonumber
    &\mathcal{L}\big(\{\theta^{t+1}_i,y^{t+1}_i\},\theta^{t+1}\big)-\mathcal{L}\big(\{\theta^{t+1}_i,y^t_i\},\theta^{t+1}\big)\\
    \nonumber
    =&\sum_{i\in\mathcal{I}}\langle  y^{t+1}_i- y^t_i,\theta^{t+1}_i-\theta^{t+1}\rangle\\
    =&\sum_{i\in\mathcal{I}}\frac{1}{\rho_i}\Vert y^{t+1}_i- y^t_i\Vert^2.
\end{align}
Using Lemma \ref{l_descent} and Assumption \ref{Lsmooth}, we have
\begin{align}
    \label{L_descent_eq2}
    \nonumber
    &\mathcal{L}\big(\{\theta^{t+1}_i,y^t_i\},\theta^{t+1}\big)-\mathcal{L}\big(\{\theta^t_i,y^t_i\},\theta^t\big)\\
    \nonumber
    =&\mathcal{L}\big(\{\theta^{t+1}_i,y^t_i\},\theta^{t+1}\big)-\mathcal{L}\big(\{\theta^t_i,y^t_i\},\theta^{t+1}\big)+\mathcal{L}\big(\{\theta^t_i,y^t_i\},\theta^{t+1}\big)-\mathcal{L}\big(\{\theta^t_i,y^t_i\},\theta^t\big)\\
    \nonumber
    =&\sum_{i\in\mathcal{I}}\left(\mathcal{L}_i(\theta^{t+1}_i,\theta^{t+1}, y^t_i)-\mathcal{L}_i(\theta^t_i,\theta^{t+1}, y^t_i)\right)+\underbrace{\mathcal{L}\big(\{\theta^t_i,y^t_i\},\theta^{t+1}\big)-\mathcal{L}\big(\{\theta^t_i,y^t_i\},\theta^t\big)}_\text{(a)}\\
    \nonumber
    \le& -\sum_{i\in\mathcal{I}}\bigg( \frac{\rho_i-8w_i\nu_i}{2}\Vert\theta^{t+1}_i-\theta^t_i\Vert^2+\frac{\rho_i-\lambda\mu_r/I}{2}\Vert\theta^{t+1}-\theta^t\Vert^2-\frac{4w_i\nu_i}{\rho^2_i}\Vert  y^{t+1}_i- y^t_i\Vert^2\\
    &-\frac{2\alpha w_i\zeta_i\beta^2_i\delta_{i,t+1}}{\rho_i}\Vert y^{t+1}_i- y^t_i\Vert-\alpha w_i\zeta_i\beta^2_i\delta_{i,t+1}\Vert \theta^{t+1}_i-\theta^t_i\Vert\bigg),
\end{align}
where the bound of (a) is derived in the similar way as in Lemma \ref{l_descent}. Combining (\ref{L_descent_eq1}) and (\ref{L_descent_eq2}), we conclude that
\begin{align}
    \label{L_descent_eq3}
    \nonumber
    &\mathcal{L}\big(\{\theta^{t+1}_i,y^{t+1}_i\},\theta^{t+1}\big)-\mathcal{L}\big(\{\theta^t_i,y^t_i\},\theta^t\big)\\
    \nonumber
    =& \mathcal{L}\big(\{\theta^{t+1}_i,y^{t+1}_i\},\theta^{t+1}\big)-\mathcal{L}\big(\{\theta^{t+1}_i,y^t_i\}_i,\theta^{t+1}\big)+\mathcal{L}\big(\{\theta^{t+1}_i,y^t_i\},\theta^{t+1}\big)-\mathcal{L}\big(\{\theta^t_i,y^t_i\},\theta^t\big)\\
    \nonumber
    =&-\sum_{i\in\mathcal{I}}\bigg[ \frac{\rho_i-8w_i\nu_i}{2}\Vert\theta^{t+1}_i-\theta^t_i\Vert^2+\frac{\rho_i-\lambda\mu_r/I}{2}\Vert\theta^{t+1}-\theta^t\Vert^2-\Big(\frac{4w_i\nu_i}{\rho^2_i}+\frac{1}{\rho_i}\Big)\Vert  y^{t+1}_i- y^t_i\Vert^2\\
    \nonumber
    &-\frac{2\alpha w_i\zeta_i\beta^2_i\delta_{i,t+1}}{\rho_i}\Vert y^{t+1}_i- y^t_i\Vert-\alpha w_i\zeta_i\beta^2_i\delta_{i,t+1}\Vert \theta^t_i-\theta^{t+1}_i\Vert\bigg]\\
    \nonumber
    \le&-\sum_{i\in\mathcal{I}}\bigg[ \frac{\rho_i-8w_i\nu_i}{2}\Vert\theta^{t+1}_i-\theta^t_i\Vert^2+\frac{\rho_i-\lambda\mu_r/I}{2}\Vert\theta^{t+1}-\theta^t\Vert^2-\Big(\frac{4w_i\nu_i}{\rho^2_i}+\frac{1}{\rho_i}\Big)\Big(2w^2_i\nu^2_i\Vert\theta^{t+1}-\theta^t\Vert^2 \\
    \nonumber
    &+ 2(\delta_{i,t}+\delta_{i,t+1})^2(\alpha w_i\zeta_i\beta^2_i)^2\Big)-\frac{2\alpha w_i\zeta_i\beta^2_i\delta_{i,t+1}}{\rho_i}\Big(w_i\nu_i\Vert\theta^{t+1}-\theta^t\Vert + (\delta_{i,t}+\delta_{i,t+1})\alpha w_i\zeta_i\beta^2_i\Big)-\alpha w_i\zeta_i\beta^2_i\delta_{i,t+1}\Vert \theta^t_i-\theta^{t+1}_i\Vert\\
    =&-\sum_{i\in\mathcal{I}}\Big(a_{i,e}\Vert\theta^{t+1}_i-\theta^t_i\Vert^2+a_{i,p}\Vert\theta^{t+1}-\theta^t\Vert^2-b^{t+1}_{i,e}\Vert\theta^{t+1}_i-\theta^t_i\Vert-b^{t+1}_{i,p}\Vert\theta^{t+1}-\theta^t\Vert-c^{t+1}_i\Big),
\end{align}
thereby completing the proof.
\end{proof}

\section{Proof of Lemma \ref{lagrangian_lowerbound}}

In the next lemma, we show that the augmented Lagrangian function $\mathcal{L}\big(\{\theta^t_i,y^t_i\},\theta^t\big)$ is lower bounded for any $t\in\mathbb{N}$.

\begin{lemma}
\label{lagrangian_lowerbound}
Given Assumption \ref{lowerbounded}-\ref{parameterassumption}, the augmented Lagrangian function defined in (\ref{eq:lagrangian}) is lower bounded by Algorithm \ref{alg}.
\end{lemma}
\begin{proof}
\label{proof_lagrangian_lowerbound}
Recall that in (\ref{dualbound_eq2})
\begin{align}
    - y^{t+1}_i=w_i\tilde{\nabla}F_i(\theta^{t+1}),
\end{align}
where $\tilde{\nabla}F_i(\theta^{t+1})= \nabla f^q_i(\phi^{t+1}_i)-\alpha g^{t+1}_i$. Besides, due to Lemma \ref{fsmoothlemma}, we can write
\begin{align}
    \label{lagrangian_lowerbound_eq1}
    \nonumber
    F_i(\theta^{t+1})\le& F_i(\theta^{t+1}_i)+\langle \nabla F_i(\theta^{t+1}_i),\theta^{t+1}-\theta^{t+1}_i\rangle+ \frac{\nu_i}{2}\Vert\theta^{t+1}_i-\theta^{t+1}\Vert^2\\
    \nonumber
    =& F_i(\theta^{t+1}_i)+\langle \nabla F_i(\theta^{t+1}),\theta^{t+1}-\theta^{t+1}_i\rangle+\langle \nabla F_i(\theta^{t+1}_i)-\nabla F_i(\theta^{t+1}),\theta^{t+1}-\theta^{t+1}_i\rangle+\frac{\nu_i}{2}\Vert\theta^{t+1}_i-\theta^{t+1}\Vert^2\\
    \le& F_i(\theta^{t+1}_i)+\langle F_i(\theta^{t+1}),\theta^{t+1}-\theta^{t+1}_i\rangle+\frac{3\nu_i}{2}\Vert\theta^{t+1}_i-\theta^{t+1}\Vert^2.
\end{align}
Based on the definition of the augmented Lagrangian function (\ref{eq:lagrangian}) and \eqref{lagrangian_lowerbound_eq1}, we can show the following observation
\begin{align}
    \nonumber
    &\mathcal{L}\big(\{\theta^{t+1}_i,y^{t+1}_i\},\theta^{t+1}\big)\\
    \nonumber
    =&\lambda D_h(\theta^{t+1},\theta_p)+\sum_{i\in\mathcal{I}}\Big(w_i F_i(\theta^{t+1}_i)
    +\langle y^{t+1} _i,\theta^{t+1}_i-\theta^{t+1}\rangle+\frac{\rho_i}{2}\Vert \theta^{t+1}_i-\theta^{t+1}\Vert^2\Big)\\
    \nonumber
    =&\lambda D_h(\theta^{t+1},\theta_p)+\sum_{i\in\mathcal{I}}\Big(w_i F_i(\theta^{t+1}_i)+\langle w_i\tilde{\nabla}F_i(\theta^{t+1}),\theta^{t+1}-\theta^{t+1}_i\rangle+\frac{\rho_i}{2}\Vert \theta^{t+1}_i-\theta^{t+1}\Vert^2\Big)\\
    \nonumber
    =&\sum_{i\in\mathcal{I}}\Big(w_i F_i(\theta^{t+1}_i)+\langle w_i\tilde{\nabla}F_i(\theta^{t+1}),\theta^{t+1}-\theta^{t+1}_i\rangle+\frac{\rho_i}{2}\Vert \theta^{t+1}_i-\theta^{t+1}\Vert^2\Big)+\lambda D_h(\theta^{t+1},\theta_p)\\
    \nonumber
    =&\sum_{i\in\mathcal{I}}\Big(w_i F_i(\theta^{t+1}_i)+\langle w_i\nabla F_i(\theta^{t+1}),\theta^{t+1}-\theta^{t+1}_i\rangle+ w_i\langle \tilde{\nabla}F_i(\theta^{t+1})-\nabla F_i(\theta^{t+1}),\theta^{t+1}-\theta^{t+1}_i\rangle\\
    \nonumber
    &+\frac{\rho_i}{2}\Vert \theta^{t+1}_i-\theta^{t+1}\Vert^2\Big)+\lambda D_h(\theta^{t+1},\theta_p)\\
    \nonumber
    \ge& \sum_{i\in\mathcal{I}}\Big(w_i F_i(\theta^{t+1}_i)+\langle w_i\nabla F_i(\theta^{t+1}),\theta^{t+1}-\theta^{t+1}_i\rangle- w_i\Vert \tilde{\nabla}F_i(\theta^{t+1})-\nabla F_i(\theta^{t+1})\Vert\cdot\Vert\theta^{t+1}-\theta^{t+1}_i\Vert\\
    \nonumber
    &+\frac{\rho_i}{2}\Vert \theta^{t+1}_i-\theta^{t+1}\Vert^2\Big)+\lambda D_h(\theta^{t+1},\theta_p)\\
    \nonumber
    \ge& \sum_{i\in\mathcal{I}}\Big(w_i F_i(\theta^{t+1})+\frac{\rho_i-3\nu_i}{2}\Vert \theta^{t+1}_i-\theta^{t+1}\Vert^2- \alpha w_i\zeta_i\beta^2_i\delta_{i,t+1}\Vert\theta^{t+1}-\theta^{t+1}_i\Vert\Big)+\lambda D_h(\theta^{t+1},\theta_p),
\end{align}
where the last inequality is derived from $\nu_i$-smoothness of $F_i$ and first-order Taylor expansion. Due to Assumption \ref{lowerbounded}, $\lambda D_h(\theta^{t+1},\theta_p)+\sum_{i\in\mathcal{I}} w_i F_i(\theta^{t+1})$ is lower bounded. According to Assumption \ref{parameterassumption}, it is easy to show that
\begin{align}
    \sum_{i\in\mathcal{I}}\Big(&\frac{\rho_i-3\nu_i}{2}\Vert \theta^{t+1}_i-\theta^{t+1}\Vert^2- \alpha w_i\zeta_i\beta^2_i\delta_{i,t+1}\Vert\theta^{t+1}-\theta^{t+1}_i\Vert\Big)>-\infty,
\end{align}
thereby completing the proof.
\end{proof}

\section{Proof of Theorem \ref{convergence}}
\label{proof:convergence}

First, we prove part (i). Note that the RHS of (\ref{L_descent_ineq}) is the sum of some independent quadratic functions of $\Vert\theta^{t+1}_i-\theta^t_i\Vert$ and $\Vert\theta^{t+1}-\theta^t\Vert$. Due to Assumption \ref{parameterassumption} and Lemma \ref{L_descent}, for each $i\in\mathcal{I}$, based on the form of roots of quadratic function, it is easy to see that there exist $\sigma^{t+1}_i$ and $\gamma^{t+1}_i$ such that
\begin{align}
\label{convergence_eq1}
    \nonumber
    \lim_{t\rightarrow\infty}\sigma^{t+1}_i=0,\\
    \lim_{t\rightarrow\infty}\gamma^{t+1}_i=0.
\end{align}
When $\Vert\theta^{t+1}_i-\theta^t_i\Vert>\sigma^{t+1}_i$, 
\begin{align}
\label{convergence_eq11}
    a_{i,e}\Vert\theta^{t+1}_i-\theta^t_i\Vert^2-b^{t+1}_{i,e}\Vert\theta^{t+1}_i-\theta^t_i\Vert-c^{t+1}_i>0;
\end{align}
and when $\Vert\theta^{t+1}-\theta^t\Vert>\gamma^{t+1}_i$, 
\begin{align}
\label{convergence_eq12}
    a_{i,p}\Vert\theta^{t+1}-\theta^t\Vert^2-b^{t+1}_{i,p}\Vert\theta^{t+1}-\theta^t\Vert>0.
\end{align}
Next, we show $\lim_{t\rightarrow\infty}\Vert\theta^{t+1}_i-\theta^t_i\Vert=0$ and $\lim_{t\rightarrow\infty}\Vert\theta^{t+1}-\theta^t\Vert=0$ by two steps. 
\begin{enumerate}[1)]
    \item Suppose that there exists $T\ge 0$ such that for all $t\ge T$, the following is true 
    \begin{align}
        \sum_{i\in\mathcal{I}}\Big(&a_{i,e}\Vert\theta^{t+1}_i-\theta^t_i\Vert^2+a_{i,p}\Vert\theta^{t+1}-\theta^t\Vert^2-b^{t+1}_{i,e}\Vert\theta^{t+1}_i-\theta^t_i\Vert-b^{t+1}_{i,p}\Vert\theta^{t+1}-\theta^t\Vert-c^{t+1}_i\Big)>0.
    \end{align}
    It follows that under Assumption \ref{parameterassumption}, using Lemma \ref{L_descent}-\ref{lagrangian_lowerbound},  $\mathcal{L}\big(\{\theta^{t+1}_i,y^{t+1}_i\},\theta^{t+1}\big)$ will monotonically decrease and converges. Thus, we obtain
    \begin{align}
        \lim_{t\rightarrow\infty}\sum_{i\in\mathcal{I}}\Big(&a_{i,e}\Vert\theta^{t+1}_i-\theta^t_i\Vert^2-b^{t+1}_{i,e}\Vert\theta^{t+1}_i-\theta^t_i\Vert+a_{i,p}\Vert\theta^{t+1}-\theta^t\Vert^2-b^{t+1}_{i,p}\Vert\theta^{t+1}-\theta^t\Vert-c^{t+1}_i\Big)=0,
    \end{align}
    which implies that $\Vert\theta^{t+1}_i-\theta^t_i\Vert$ and $\Vert\theta^{t+1}-\theta^t\Vert$ converge to the positive roots of corresponding quadratic functions, i.e., LHS of (\ref{convergence_eq11}) and (\ref{convergence_eq12}), otherwise the limitation will not be 0. Due to (\ref{convergence_eq1}), the positive roots of the above quadratic function converge to 0, which implies
    \begin{align}
        \label{convergence_eq2}
        &\lim_{t\rightarrow\infty}\Vert\theta^{t+1}_i-\theta^t_i\Vert=0,~\forall~i\in\mathcal{I},\\
        \label{convergence_eq3}
        &\lim_{t\rightarrow\infty}\Vert\theta^{t+1}-\theta^t\Vert=0.
    \end{align}
    By Lemma \ref{dualbound} and (\ref{eq:update_y}), we can show that
    \begin{align}
        \label{convergence_eq4}
        &\lim_{t\rightarrow\infty}\Vert y^{t+1}_i- y^t_i\Vert=0,~\forall~i\in\mathcal{I},\\
        \label{convergence_eq5}
        &\lim_{t\rightarrow\infty}\Vert\theta^{t+1}_i-\theta^{t+1}\Vert=0,~\forall~i\in\mathcal{I}.
    \end{align}
    \item Suppose that there exists a sequence $\{t_j\mid j\in\mathbb{N}\}$ such that
    \begin{align}
        \label{convergence_eq14}
        \sum_{i\in\mathcal{I}}\Big(&a_{i,e}\Vert\theta^{t_j+1}_i-\theta^{t_j}_i\Vert^2+a_{i,p}\Vert\theta^{t_j+1}-\theta^{t_j}\Vert^2-b^{t_j+1}_{i,e}\Vert\theta^{t_j+1}_i-\theta^{t_j}_i\Vert-b^{t_j+1}_{i,p}\Vert\theta^{t_j+1}-\theta^{t_j}\Vert-c^{t_j+1}_i\Big)\le0.
    \end{align}
    Due to Assumption \ref{parameterassumption}, the minimum value of the above quadratic function converges to 0, which implies
    \begin{align}
        \lim_{t\rightarrow\infty}\sum_{i\in\mathcal{I}}\Big(&a_{i,e}\Vert\theta^{t_j+1}_i-\theta^{t_j}_i\Vert^2+a_{i,p}\Vert\theta^{t_j+1}-\theta^{t_j}\Vert^2-b^{t_j+1}_{i,e}\Vert\theta^{t_j+1}_i-\theta^{t_j}_i\Vert-b^{t_j+1}_{i,p}\Vert\theta^{t_j+1}-\theta^{t_j}\Vert-c^{t_j+1}_i\Big)=0.
    \end{align}
    Similar to (\ref{convergence_eq2}) and (\ref{convergence_eq3}), we have 
    \begin{align}
        &\lim_{t\rightarrow\infty}\Vert\theta^{t_j+1}_i-\theta^{t_j}_i\Vert=0,~\forall~i\in\mathcal{I},\\
        &\lim_{t\rightarrow\infty}\Vert\theta^{t_j+1}-\theta^{t_j}\Vert=0.
    \end{align}
    We also define a nontrivial sequence $\{t_q\mid q\in\mathbb{N}\}\triangleq\mathbb{N}-\{t_j\mid j\in\mathbb{N}\}$. Note that
    \begin{align}
        \label{convergence_eq15}
        \sum_{i\in\mathcal{I}}\Big(&a_{i,e}\Vert\theta^{t_q+1}_i-\theta^{t_q}_i\Vert^2+a_{i,p}\Vert\theta^{t_q+1}-\theta^{t_q}\Vert^2-b^{t_q+1}_{i,e}\Vert\theta^{t_q+1}_i-\theta^{t_q}_i\Vert-b^{t_q+1}_{i,p}\Vert\theta^{t_q+1}-\theta^{t_q}\Vert-c^{t_q+1}_i\Big)>0.
    \end{align}
    Similar to 1), we have
    \begin{align}
        &\lim_{t\rightarrow\infty}\Vert\theta^{t_q+1}_i-\theta^{t_q}_i\Vert=0,~\forall~i\in\mathcal{I},\\
        &\lim_{t\rightarrow\infty}\Vert\theta^{t_q+1}-\theta^{t_q}\Vert=0.
    \end{align}
    Based on the above observations, for any $\eta>0$, there exists $k\ge 0$ such that when $j>k$ and $q>k$, the following holds true 
    \begin{align}
        &\Vert\theta^{t_j+1}_i-\theta^{t_j}_i\Vert\le\eta,~\forall~i\in\mathcal{I},\\
        &\Vert\theta^{t_j+1}-\theta^{t_j}\Vert\le\eta,\\
        &\Vert\theta^{t_q+1}_i-\theta^{t_q}_i\Vert\le\eta,~\forall~i\in\mathcal{I},\\
        &\Vert\theta^{t_q+1}-\theta^{t_q}\Vert\le\eta.
    \end{align}
    Thus, for any $t>t_k$, we can write
    \begin{align}
        &\Vert\theta^{t+1}_i-\theta^t_i\Vert\le\eta,~\forall~i\in\mathcal{I},\\
        &\Vert\theta^{t+1}-\theta^t\Vert\le\eta,
    \end{align}
    which implies that 
    \begin{align}
        &\lim_{t\rightarrow\infty}\Vert\theta^{t+1}_i-\theta^t_i\Vert=0,~\forall~i\in\mathcal{I},\\
        &\lim_{t\rightarrow\infty}\Vert\theta^{t+1}-\theta^t\Vert=0,
    \end{align}
    thereby (\ref{convergence_eq2})-(\ref{convergence_eq5}) hold.
\end{enumerate}
Using the optimality condition of (\ref{eq:update_thetai_inex}) leads to
\begin{align}
    \label{convergence_eq6}
    w_i\tilde{\nabla}F_i(\theta^{t+1})+ y^{t+1}_i=0,
\end{align}
where $\tilde{\nabla}F_i(\theta^{t+1})\triangleq \nabla f^q_i(\phi^{t+1}_i)-\alpha g^{t+1}_i$. For each $i\in\mathcal{I}$, we derive an upper bound of $\Vert\nabla_{\theta^{t+1}_i}\mathcal{L}\big(\{\theta^{t+1}_i,y^{t+1}_i\},\theta^{t+1}\big)\Vert$ as
\begin{align}
    \label{convergence_eq7}
    \nonumber
    &\Vert\nabla_{\theta^{t+1}_i}\mathcal{L}\big(\{\theta^{t+1}_i,y^{t+1}_i\},\theta^{t+1}\big)\Vert\\
    \nonumber
    =&\Vert w_i\nabla F_i(\theta^{t+1}_i)+ y^{t+1}_i+\rho_i(\theta^{t+1}_i-\theta^{t+1})\Vert\\
    \nonumber
    \le&\Vert w_i\nabla  F_i(\theta^{t+1}_i)-w_i\tilde{\nabla}F_i(\theta^{t+1})\Vert+\rho_i\Vert\theta^{t+1}_i-\theta^{t+1}\Vert+\Vert w_i\tilde{\nabla}F_i(\theta^{t+1})+ y^{t+1}_i\Vert\\
    \nonumber
    \le& \Vert w_i\nabla  F_i(\theta^{t+1}_i)-w_i\nabla F_i(\theta^{t+1})\Vert+\Vert w_i\nabla F_i(\theta^{t+1})-w_i\tilde{\nabla}F_i(\theta^{t+1})\Vert+\rho_i\Vert\theta^{t+1}_i-\theta^{t+1}\Vert\\
    \le&\Vert w_i\nabla  F_i(\theta^{t+1}_i)-w_i\nabla F_i(\theta^{t+1})\Vert+\rho_i\Vert\theta^{t+1}_i-\theta^{t+1}\Vert+\alpha w_i\zeta_i\beta^2_i\delta_{i,t+1}.
\end{align}
Taking limitation of $t\rightarrow\infty$ on both sides of (\ref{convergence_eq7}), and using Assumption \ref{parameterassumption} and (\ref{convergence_eq5}) yields
\begin{align}
    \label{convergence_eq8}
    \Vert\nabla_{\theta^*_i}\mathcal{L}\big(\{\theta^*_i,y^*_i\},\theta^*\big)\Vert=0,~\forall~i\in\mathcal{I}.
\end{align}
Note that
\begin{align}
\label{convergence_eq9}
    \Vert\nabla_{\theta^{t+1}}\mathcal{L}\big(\{\theta^{t+1}_i,y^{t+1}_i\},\theta^{t+1}\big)\Vert\le \sum_{i\in\mathcal{I}}\Big(\rho_i\Vert\theta^{t+1}_i-\theta^t_i\Vert+w_i\nu_i\Vert\theta^{t+1}-\theta^t\Vert+2\alpha w_i\zeta_i\beta^2_i\delta_{i,t}\Big).
\end{align}
Similarly, we obtain
\begin{align}
\label{convergence_eq10}
    \nabla_{\theta^*}\mathcal{L}\big(\{\theta^*_i,y^*_i\},\theta^*\big)=0.
\end{align}
Finally, we bound $\big\Vert \sum_{i\in\mathcal{I}}w_i\nabla F_i(\theta^{t+1})+\lambda \nabla D_h(\theta^{t+1},\theta_p)\big\Vert$ by
\begin{align}
    \label{convergence_eq13}
    \nonumber
    &\Big\Vert \sum_{i\in\mathcal{I}}w_i\nabla F_i(\theta^{t+1})+\lambda \nabla D_h(\theta^{t+1},\theta_p)\Big\Vert\\
    \nonumber
    \le& \Big\Vert \sum_{i\in\mathcal{I}}w_i\nabla F_i(\theta^{t+1})+\lambda \nabla D_h(\theta^{t+1},\theta_p)-\sum_{i\in\mathcal{I}}\nabla_{\theta^{t+1}_i}\mathcal{L}\big(\{\theta^{t+1}_i,y^t_i\},\theta^{t+1}\big)\Big\Vert+\Big\Vert\sum_{i\in\mathcal{I}}\nabla_{\theta^{t+1}_i}\mathcal{L}\big(\{\theta^{t+1}_i,y^t_i\},\theta^{t+1}\big) \Big\Vert\\
    \nonumber
    =&\Big\Vert\sum_{i\in\mathcal{I}}w_i\big(\nabla F_i(\theta^{t+1})-\nabla F_i(\theta^{t+1}_i)\big)+\lambda \nabla D_h(\theta^{t+1},\theta_p)-\sum_{i\in\mathcal{I}}\big(y^t_i+\rho_i(\theta^{t+1}_i-\theta^{t+1})\big)\Big\Vert\\
    \nonumber
    &+\Big\Vert\sum_{i\in\mathcal{I}}\nabla_{\theta^{t+1}_i}\mathcal{L}\big(\{\theta^{t+1}_i,y^t_i\},\theta^{t+1}\big) \Big\Vert\\
    \le& \sum_{i\in\mathcal{I}}w_i\nu_i\Vert\theta^{t+1}-\theta^{t+1}_i\Vert+\Vert \nabla_{\theta^{t+1}}\mathcal{L}\big(\{\theta^{t+1}_i,y^t_i\},\theta^{t+1}\big)\Vert+\sum_{i\in\mathcal{I}}\Vert\nabla_{\theta^{t+1}_i}\mathcal{L}\big(\{\theta^{t+1}_i,y^t_i\},\theta^{t+1}\big)\Vert.
\end{align}
Taking limitation of (\ref{convergence_eq13}) by $t\rightarrow\infty$, and combining (\ref{convergence_eq5}), (\ref{convergence_eq8}) and (\ref{convergence_eq10}) yield part (i).

Next, we prove part (ii), Summing up the Inequality (\ref{L_descent_ineq}) from $t=0$ to $T$ and taking a limitation on $T$, there exist some positive constants $a_2$ and $a_1$ corresponding to $\rho_i$ such that
\begin{align}
    \sum^\infty_{t=0}z^t\le \mathcal{L}\big(\{\theta^0_i,y^0_i\},\theta^0\big)-\mathcal{L}\big(\{\theta^*_i,y^*_i\},\theta^*\big)<\infty,
\end{align}
where $z^t$ is denoted by 
\begin{align}
    z^t\triangleq& \underbrace{a_2\sum_{i\in\mathcal{I}}\Big(\Vert\theta^{t+1}_i-\theta^t_i\Vert^2+\Vert\theta^{t+1}-\theta^t\Vert^2\Big)}_{z^t_2}-\underbrace{a_1\sum_{i\in\mathcal{I}}\Big(\delta_{i,t+1}\big(\Vert\theta^{t+1}_i-\theta^t_i\Vert+\Vert\theta^{t+1}-\theta^t\Vert\big)+\delta^2_{i,t}\Big)}_{z^t_1}\\
    =&z^t_2-z^t_1.
\end{align}
$z^t_2$, $z^t_1$ are denoted as the first and second sum terms, respectively. Due to Assumption \ref{parameterassumption} and Theorem \ref{convergence}, it is easy to see that there exists some positive constant $a_3$ such that the following holds true
\begin{align}
    \sum^{\infty}_{t=0}z^t_1    &=a_1\sum^{\infty}_{t=0}\sum_{i\in\mathcal{I}}\Big(\delta_{i,t+1}\big(\Vert\theta^{t+1}_i-\theta^t_i\Vert+\Vert\theta^{t+1}-\theta^t\Vert\big)+\delta^2_{i,t}\Big)\\
    &\le a_1\sum_{i\in\mathcal{I}}\Big(\sum^{\infty}_{t=0}2\delta_{i,t+1}+\sum^{\infty}_{t=0}\delta^2_{i,t}\Big)+a_3\\
    &< \infty.
\end{align}
Hence, we have
\begin{align}
    \sum^{\infty}_{t=1}z^t_2\le b<\infty,~\text{for some constant}~b>0.
\end{align}
Due to $\sum^{\infty}_{t=0}z^t_1<\infty$, it is easy to see that the augmented Lagrangian function is upper bounded and $\theta^t$ is finite, which shows $\{\theta^t\}$ has at least one limit point. Denoting $T^2(\epsilon)\triangleq\min\{t\mid \Vert\theta^{t+1}-\theta^t\Vert^2\le\epsilon,t\ge0\}$ and $T^2_i(\epsilon)\triangleq\min\{t\mid \Vert\theta^{t+1}_i-\theta^t_i\Vert^2\le\epsilon,t\ge0\}$, then we can write
\begin{align}
    \label{complexity_eq1}
    &a_2 T^2(\epsilon)\epsilon\le\sum^{\infty}_{t=1}z^t_2\le b,\\
    \label{complexity_eq2}
    &a_2 T^2_i(\epsilon)\epsilon\le\sum^{\infty}_{t=1}z^t_2\le b.
\end{align}
That is, $T^2(\epsilon)=\mathcal{O}(1/\epsilon)$ and $T^2_i(\epsilon)=\mathcal{O}(1/\epsilon)$ hold. Further, we denote $T(\epsilon)\triangleq\min\{t\mid\Vert\theta^{t+1}-\theta^t\Vert\le\epsilon,t\ge0\}$ and $T_i(\epsilon)\triangleq\min\{t\mid\Vert\theta^{t+1}_i-\theta^t_i\Vert\le\epsilon,t\ge0\}$. Based on (\ref{complexity_eq1}) and (\ref{complexity_eq2}), we have $T(\epsilon)=\mathcal{O}(1/\epsilon^2)$ and $T_i(\epsilon)=\mathcal{O}(1/\epsilon^2)$. Due to Assumption \ref{parameterassumption}, combining (\ref{dualbound_ineq_1}) and (\ref{convergence_eq7}) yields
\begin{align}
    \Vert\nabla_{\theta^{t+1}_i}\mathcal{L}\big(\{\theta^{t+1}_i,y^{t+1}_i\},\theta^{t+1}\big)\Vert
    \le\frac{(w_i\nu_i+\rho_i)w_i\nu_i}{\rho_i}\cdot\Vert\theta^{t+1}-\theta^t\Vert+\frac{(2w_i\nu_i+3\rho_i)\alpha w_i\zeta_i\beta^2_i}{\rho_i}\cdot\delta_{i,t}.
\end{align}
Similarly, it is easy to see that the convergence rate of $\delta_{i,t}$ is $\mathcal{O}(1/\epsilon)$. Therefore, for any $\epsilon>0$, Algorithm \ref{alg} finds a point $(\{\theta_i\},\theta, y)$ with $\Vert\nabla_{\theta_i}\mathcal{L}\big(\{\theta_i,y_i\},\theta\big)\Vert\le\epsilon$, after at most $\mathcal{O}(1/\epsilon^2)$. In the same way, it can be shown that 
\begin{align}
     \Vert\nabla_{\theta^{t+1}}\mathcal{L}\big(\{\theta^{t+1}_i,y^{t+1}_i\},\theta^{t+1}\big)\Vert\le \sum_{i\in\mathcal{I}}\Big(\rho_i\Vert\theta^{t+1}_i-\theta^t_i\Vert+w_i\nu_i\Vert\theta^{t+1}-\theta^t\Vert+2\alpha w_i\zeta_i\beta^2_i\delta_{i,t}\Big),
\end{align}
which implies that $\Vert\nabla_{\theta^{t+1}}\mathcal{L}\big(\{\theta^{t+1}_i,y^{t+1}_i\},\theta^{t+1}\big)\Vert$ has the same communication complexity as $\Vert\nabla_{\theta^{t+1}_i}\mathcal{L}\big(\{\theta^{t+1}_i,y^{t+1}_i\},\theta^{t+1}\big)\Vert$. Moreover, it is easy to show that $\Vert\nabla_{y_i}\mathcal{L}\big(\{\theta_i,y_i\},\theta\big)\Vert$ also the same complexity. Combining with (\ref{convergence_eq13}), it completes the proof.

\section{Proof of Corollary \ref{coro:data_impact}}
\label{proof:coro}
First, we prove the following result
\begin{align}
    \label{eq:coro_1}
    \mathbb{E}\Big\{\big\Vert \nabla L_i\big(\theta-\alpha\nabla L_i(\theta)\big)-\nabla L_i\big(\theta-\alpha\nabla L_i(\theta,\mathcal{D}^s_i),\mathcal{D}^q_i\big)\big\Vert\Big\}\le \frac{\alpha\mu_i\sigma^g_i}{\sqrt{D^s_i}}+\frac{\sigma^g_i}{\sqrt{D^q_i}}.
\end{align}
We can write
\begin{align}
    \label{eq:eg_1}
    &\mathbb{E}\Big\{\big\Vert \nabla L_i\big(\theta-\alpha\nabla L_i(\theta)\big)-\nabla L_i\big(\theta-\alpha\nabla L_i(\theta,\mathcal{D}^s_i),\mathcal{D}^q_i\big)\big\Vert\Big\}\nonumber\\
    \le&\underbrace{\mathbb{E}\Big\{\big\Vert \nabla L_i\big(\theta-\alpha\nabla L_i(\theta)\big)-\nabla L_i\big(\theta-\alpha\nabla L_i(\theta,\mathcal{D}^s_i)\big)\big\Vert\Big\}}_\text{(a)}+\underbrace{\mathbb{E}\Big\{\big\Vert \nabla L_i\big(\theta-\alpha\nabla L_i(\theta,\mathcal{D}^s_i)\big)-\nabla L_i\big(\theta-\alpha\nabla L_i(\theta,\mathcal{D}^s_i),\mathcal{D}^q_i\big)\big\Vert\Big\}}_\text{(b)}.
\end{align}
For (b), due to Assumption \ref{bounded_var}, we have:
\begin{align}
    \label{eq:eg_2}
    \text{(c)}\le&\sqrt{\mathbb{E}\Big\{\big\Vert \nabla L_i(\phi_i)-\frac{1}{D^q_i}\sum^{D^q_i}_{j=1}\nabla l\big(\phi_i,(\mathbf{x}^j_i,\mathbf{y}^j_i)\big)\big\Vert^2\Big\}}\nonumber\\
    =&\sqrt{\frac{1}{(D^q_i)^2}\sum^{D^q_i}_{j=1}\mathbb{E}\Big\{\big\Vert \nabla L_i(\phi_i)-\nabla l\big(\phi_i,(\mathbf{x}^j_i,\mathbf{y}^j_i)\big)\big\Vert^2\Big\}}\nonumber\\
    \le&\frac{\sigma^g_i}{\sqrt{D^q_i}},
\end{align}
where $\phi_i=\theta-\alpha\nabla L_i(\theta,\mathcal{D}^s_i)$. For (a), based on Assumption \ref{Lsmooth}, it can be bounded by:
\begin{align}
    \label{eq:eg_3}
    \text{(b)}\le\mathbb{E}\Big\{\alpha\mu_i\big\Vert \nabla L_i(\theta,\mathcal{D}^s_i)-\nabla L_i(\theta)\big\Vert\Big\}
    \le\frac{\alpha\mu_i\sigma^g_i}{\sqrt{D^s_i}}.
\end{align}
Plugging \eqref{eq:eg_2} and \eqref{eq:eg_3} into \eqref{eq:eg_1}, \eqref{eq:coro_1} holds. Based on \eqref{eq:coro_1}, the following holds
\begin{align}
    \label{eq:convergence_expected}
     &\mathbb{E}\Big\{\big\Vert \sum_{i\in\mathcal{I}} w_i \nabla L_i\big(\theta_{\epsilon}-\alpha\nabla L_i(\theta_{\epsilon})\big)+\lambda \nabla D_h(\theta_{\epsilon},\theta_p)\big\Vert\Big\}\nonumber\\
     \le& \sum_{i\in\mathcal{I}}w_i \mathbb{E}\Big\{\big\Vert \nabla L_i\big(\theta_{\epsilon}-\alpha\nabla L_i(\theta_{\epsilon})\big)-\nabla L_i\big(\theta_{\epsilon}-\alpha\nabla L_i(\theta_{\epsilon},\mathcal{D}^s_i),\mathcal{D}^q_i\big)\big\Vert\Big\}+\mathbb{E}\Big\{\big\Vert \sum_{i\in\mathcal{I}} w_i \nabla L_i\big(\theta_{\epsilon}-\alpha\nabla L_i(\theta_{\epsilon},\mathcal{D}^s_i),\mathcal{D}^q_i\big)+\lambda \nabla D_h(\theta_{\epsilon},\theta_p)\big\Vert\Big\}\nonumber\\
     \le&\epsilon+ \sum_{i\in\mathcal{I}}w_i\sigma^g_i\Big(\frac{\alpha\mu_i}{\sqrt{D^s_i}}+\frac{1}{\sqrt{D^q_i}}\Big),
\end{align}
thereby completing the proof.

\section{Proof of Theorem \ref{performance}}
\label{proof:performance}

Let $\theta_{\epsilon}$ denote the $\epsilon$-FOSP obtained by Algorithm \ref{alg}, which satisfies that
\begin{align}
    \label{performance_eq2}
    &\Vert \sum_{i\in\mathcal{I}}w_i\nabla F_i(\theta_{\epsilon})+ \lambda \nabla D_h(\theta_{\epsilon},\theta_p)\Vert\le\epsilon,
\end{align}
for some $\epsilon>0$. Then, for the learned model parameter $\theta_{\epsilon}$, $\mathbb{E}\big\{\Vert \nabla F_m(\theta_{\epsilon})+\lambda\nabla D_h(\theta_{\epsilon},\theta_p)\Vert\big\}$ can be upper bounded by
\begin{align}
    \label{performance_eq4}
    \nonumber
    \mathbb{E}\big\{\Vert \nabla F_m(\theta_{\epsilon})+\lambda\nabla D_h(\theta_{\epsilon},\theta_p)\Vert\big\}=&\mathbb{E}\Big\{\big\Vert\sum_{i\in\mathcal{I}}w_i\nabla F_i(\theta_{\epsilon})+ \lambda \nabla D_h(\theta_{\epsilon},\theta_p)+\sum_{i\in\mathcal{I}}w_i\big(\nabla F_m(\theta_{\epsilon})-\nabla F_i(\theta_{\epsilon})\big)\big\Vert\Big\}\\
    \le&\epsilon+\mathbb{E}\big\{\underbrace{\Vert\nabla F_m(\theta_{\epsilon})-\sum_{i\in\mathcal{I}}w_i\nabla F_i(\theta_{\epsilon})\Vert}_{\text{(a)}}\big\}.
\end{align}
Due to Assumption \ref{similarity}, for $i\in\mathcal{I}\cup\{m\}$ and $\mathcal{D}_i$ with respect to $P_i$, we can write
\begin{align}
    \label{performance_9}
    \mathbb{E}\big\{\big\Vert \nabla L_i(\theta,\mathcal{D}_i)-\nabla L_i(\theta)\big\Vert\big\}\le\frac{\sigma^g_i}{\sqrt{D_i}}.
\end{align}
Based on (\ref{performance_9}), observe that
\begin{align}
    \label{performance_10}
    \nonumber
    &\mathbb{E}\big\{\Vert \nabla L_m(\theta,\mathcal{D}_m)-\nabla L_i(\theta,\mathcal{D}_i)\Vert\big\}\\
    \nonumber
    \le&\mathbb{E}\big\{\Vert \nabla L_m(\theta,\mathcal{D}_m)-\nabla L_m(\theta)\Vert\big\}+\mathbb{E}\big\{\Vert \nabla L_m(\theta)-\nabla L_i(\theta)\Vert\big\}+\mathbb{E}\big\{\Vert \nabla L_i(\theta,\mathcal{D}_i)-\nabla L_i(\theta)\Vert\big\}\\
    \le&\psi^g_i+\frac{\sigma^g_i}{\sqrt{D_m}}+\frac{\sigma^g_i}{\sqrt{D_i}}.
\end{align}
Similarly, we can show that
\begin{align}
    \label{performance_11}
    \mathbb{E}\big\{\Vert \nabla^2 L_m(\theta,\mathcal{D}_m)-\nabla^2 L_i(\theta,\mathcal{D}_i)\Vert\big\}\le\psi^h_i+\frac{\sigma^h_i}{\sqrt{D_m}}+\frac{\sigma^h_i}{\sqrt{D_i}}.
\end{align}
Thus, for (a), we obtain
\begin{align}
    \label{performance_eq5}
    \nonumber
    &\Big\Vert\nabla F_m(\theta)-\sum_{i\in\mathcal{I}}w_i\nabla F_i(\theta)\Big\Vert\\
    \nonumber
    =&\Big\Vert \nabla_\theta L_m\big(\theta-\alpha\nabla L_m(\theta,\mathcal{D}^{s}_m),\mathcal{D}^q_m\big)-\sum_{i\in\mathcal{I}}w_i\nabla_\theta L_i\big(\theta-\alpha\nabla L_i(\theta,\mathcal{D}^{s}_i),\mathcal{D}^q_i\big)\Big\Vert\\
    \nonumber
    \le&\underbrace{\Big\Vert\nabla L_m\big(\theta-\alpha\nabla L_m(\theta,\mathcal{D}^{s}_m),\mathcal{D}^q_m\big)-\sum_{i\in\mathcal{I}}w_i\nabla L_i\big(\theta-\alpha\nabla L_i(\theta,\mathcal{D}^{s}_i),\mathcal{D}^q_i\big)\Big\Vert}_\text{(b)}\\
    &+\underbrace{\alpha\Big\Vert\nabla^2 L_m(\theta,\mathcal{D}^{s}_m)\nabla L_m\big(\theta-\alpha\nabla L_m(\theta,\mathcal{D}^{s}_m),\mathcal{D}^q_m\big)-\sum_{i\in\mathcal{I}}w_i\nabla^2 L_i(\theta,\mathcal{D}^{s}_i)\nabla L_i\big(\theta-\alpha\nabla L_i(\theta,\mathcal{D}^{s}_i),\mathcal{D}^q_i\big)\Big\Vert}_\text{(c)}.
\end{align}
Based on Assumption \ref{Lsmooth} and \ref{similarity}, we have
\begin{align}
    \label{performance_eq6}
    \nonumber
    \mathbb{E}\{\text{(b)}\}\le&\mathbb{E}\Big\{\big\Vert\nabla L_m\big(\theta-\alpha\nabla L_m(\theta,\mathcal{D}^{s}_m),\mathcal{D}^q_m\big)-\sum_{i\in\mathcal{I}}w_i\nabla L_i\big(\theta-\alpha\nabla L_m(\theta,\mathcal{D}^{s}_m),\mathcal{D}^q_i\big)\big\Vert\\
    \nonumber
    &+\big\Vert\sum_{i\in\mathcal{I}}w_i\nabla L_i\big(\theta-\alpha\nabla L_m(\theta,\mathcal{D}^{s}_m),\mathcal{D}^q_i\big)-\sum_{i\in\mathcal{I}}w_i\nabla L_i\big(\theta-\alpha\nabla L_i(\theta,\mathcal{D}^{s}_i),\mathcal{D}^q_i\big)\big\Vert\Big\}\\
    \le& \underbrace{\sum_{i\in\mathcal{I}}w_i\Bigg(\psi^g_i+\frac{\sigma^g_i}{\sqrt{D^{q}_m}}+\frac{\sigma^g_i}{\sqrt{D^{q}_i}}\Bigg)+\alpha\mu\sum_{i\in\mathcal{I}}w_i\Bigg(\psi^g_i+\frac{\sigma^g_i}{\sqrt{D^{s}_m}}+\frac{\sigma^g_i}{\sqrt{D^{s}_i}}\Bigg)}_\text{(d)}.
\end{align}
Similarly, for (c), it follows that
\begin{align}
    \label{performance_eq7}
    \nonumber
    \mathbb{E}\{\text{(c)}\}
    \le&\alpha\mathbb{E}\Big\{\big\Vert\nabla^2 L_m(\theta,\mathcal{D}^{s}_m)\nabla L_m\big(\theta-\alpha\nabla L_m(\theta,\mathcal{D}^{s}_m),\mathcal{D}^q_m\big)\\
    \nonumber
    &-\sum_{i\in\mathcal{I}}w_i\nabla^2 L_i(\theta,\mathcal{D}^{s}_i)\nabla L_m\big(\theta-\alpha\nabla L_m(\theta,\mathcal{D}^{s}_m),\mathcal{D}^q_m\big)\big\Vert\Big\}\\
    \nonumber
    &+\alpha\mathbb{E}\Big\{\big\Vert\sum_{i\in\mathcal{I}}w_i\nabla^2 L_i(\theta,\mathcal{D}^{s}_i)\nabla L_m\big(\theta-\alpha\nabla L_m(\theta,\mathcal{D}^{s}_m),\mathcal{D}^q_m\big)\\
    \nonumber
    &-\sum_{i\in\mathcal{I}}w_i\nabla^2 L_i(\theta,\mathcal{D}^{s}_i)\nabla L_i\big(\theta-\alpha\nabla L_i(\theta,\mathcal{D}^{s}_i),\mathcal{D}^q_i\big)\big\Vert\Big\}\\
    \le& \alpha\beta_m\sum_{i\in\mathcal{I}}w_i\Bigg(\psi^h_i+\frac{\sigma^h_i}{\sqrt{D^{s}_m}}+\frac{\sigma^h_i}{\sqrt{D^{s}_i}}\Bigg)+\alpha\mu\cdot\text{(d)}.
\end{align}
Plugging (\ref{performance_eq6}) and (\ref{performance_eq7}) in (\ref{performance_eq5}) yields
\begin{align}
    \label{performance_eq8}
    \nonumber
    \mathbb{E}\{\text{(a)}\}\le&\alpha\beta_m\sum_{i\in\mathcal{I}}w_i\Bigg(\psi^h_i+\frac{\sigma^h_i}{\sqrt{D^{s}_m}}+\frac{\sigma^h_i}{\sqrt{D^{s}_i}}\Bigg)+(\alpha\mu+1)(\alpha\mu)\sum_{i\in\mathcal{I}}w_i\Bigg(\psi^g_i+\frac{\sigma^g_i}{\sqrt{D^{q}_m}}+\frac{\sigma^g_i}{\sqrt{D^{q}_i}}\Bigg)\\
    \nonumber
    &+(\alpha\mu+1)\sum_{i\in\mathcal{I}}w_i\Bigg(\psi^g_i+\frac{\sigma^g_i}{\sqrt{D^{s}_m}}+\frac{\sigma^g_i}{\sqrt{D^{s}_i}}\Bigg)\\
    \nonumber
    =&\alpha\beta_m\sum_{i\in\mathcal{I}}w_i\psi^h_i+(\alpha\mu+1)^2\sum_{i\in\mathcal{I}}w_i\psi^g_i+\alpha\beta_m\sum_{i\in\mathcal{I}}w_i\Bigg(\frac{\sigma^h_i}{\sqrt{D^{s}_m}}+\frac{\sigma^h_i}{\sqrt{D^{s}_i}}\Bigg)\\
    &+(\alpha\mu+1)(\alpha\mu)\sum_{i\in\mathcal{I}}w_i\Bigg(\frac{\sigma^g_i}{\sqrt{D^{q}_m}}+\frac{\sigma^g_i}{\sqrt{D^{q}_i}}\Bigg)+(\alpha\mu+1)\sigma^g_i\sum_{i\in\mathcal{I}}w_i\Bigg(\frac{\sigma^g_i}{\sqrt{D^{s}_m}}+\frac{\sigma^g_i}{\sqrt{D^{s}_i}}\Bigg).
\end{align}
Therefore, plugging (\ref{performance_eq8}) in (\ref{performance_eq4}), we obtain
\begin{align}
    \mathbb{E}\big\{\Vert \nabla F_m(\theta_{\epsilon})+\lambda\nabla D_h(\theta_{\epsilon},\theta_p)\Vert\big\}\le&\epsilon+\alpha\beta_m\sum_{i\in\mathcal{I}}w_i\psi^h_i+(\alpha\mu+1)^2\sum_{i\in\mathcal{I}}w_i\psi^g_i+\alpha\beta_m\sum_{i\in\mathcal{I}}w_i\Bigg(\frac{\sigma^h_i}{\sqrt{D^{s}_m}}+\frac{\sigma^h_i}{\sqrt{D^{s}_i}}\Bigg)\\
    \nonumber
    &+(\alpha\mu+1)(\alpha\mu)\sum_{i\in\mathcal{I}}w_i\Bigg(\frac{\sigma^g_i}{\sqrt{D^{q}_m}}+\frac{\sigma^g_i}{\sqrt{D^{q}_i}}\Bigg)+(\alpha\mu+1)\sum_{i\in\mathcal{I}}w_i\Bigg(\frac{\sigma^g_i}{\sqrt{D^{s}_m}}+\frac{\sigma^g_i}{\sqrt{D^{s}_i}}\Bigg),
\end{align}
thereby completing the proof. 

\section{Proof of Lemma \ref{lem:forgetting}}
\label{proof:forgetting}
From Corollary \ref{coro:data_impact}, the following holds
\begin{align}
    \nonumber
    &\mathbb{E}\Big\{\Big\Vert \sum_{i\in\mathcal{I}}w_i \nabla L_i\big(\theta_{\epsilon}-\alpha\nabla L_i(\theta_{\epsilon})\big)+\lambda \nabla D_h(\theta_{\epsilon},\theta_p)\Big\Vert\Big\}\le\epsilon+\sum_{i\in\mathcal{I}}w_i\sigma^g_i\Bigg(\frac{\alpha\mu_i}{\sqrt{D^{s}_i}}+\frac{1}{\sqrt{D^{q}_i}}\Bigg).
\end{align}
Thus, based on Assumption \ref{Lsmooth}, we can obtain
\begin{align}
    \label{eq:forgetting_l2_1}
     \lambda\Vert \nabla D_h(\theta_{\epsilon},\theta_p)\Vert=&\Big\Vert\sum_{i\in\mathcal{I}}w_i \nabla L_i\big(\theta_{\epsilon}-\alpha\nabla L_i(\theta_{\epsilon})\big)+\lambda \nabla D_h(\theta_{\epsilon},\theta_p)-\sum_{i\in\mathcal{I}}w_i \nabla L_i\big(\theta_{\epsilon}-\alpha\nabla L_i(\theta_{\epsilon})\big)\Big\Vert\nonumber\\
    \le&\Big\Vert\sum_{i\in\mathcal{I}}w_i \nabla L_i\big(\theta_{\epsilon}-\alpha\nabla L_i(\theta_{\epsilon})\big)+\lambda \nabla D_h(\theta_{\epsilon},\theta_p)\Big\Vert+\Big\Vert\sum_{i\in\mathcal{I}}w_i \nabla L_i\big(\theta_{\epsilon}-\alpha\nabla L_i(\theta_{\epsilon})\big)\Big\Vert\nonumber\\
    \le&\epsilon+\sum_{i\in\mathcal{I}}w_i\Bigg(\beta_i+\frac{\alpha\mu_i\sigma^g_i}{\sqrt{D^{s}_i}}+\frac{\sigma^g_i}{\sqrt{D^{q}_i}}\Bigg).
\end{align}
Due to the convexity of $D_h(\cdot,\theta_p)$ and $D_h(\theta_p,\theta_p)=0$, we have
\begin{align}
    \mathbb{E}\big\{D_h(\theta,\theta_p)\big\}\le\frac{1}{\lambda}\left(\epsilon+\sum_{i\in\mathcal{I}}w_i\left(\beta_i+\frac{\alpha\mu_i\sigma^g_i}{\sqrt{D^s_i}}+\frac{\sigma^g_i}{\sqrt{D^q_i}}\right)\right)\Vert\theta_{\epsilon}-\theta_p\Vert.
\end{align}
Equation \eqref{eq:forgetting_general_strong_ex} can be directly derived via \eqref{eq:strongly_convex_primal_ex}.

\end{document}